\documentclass{article} 
\usepackage{iclr2027_conference,times}

\usepackage{amsmath,amsfonts,bm}

\def\eqref#1{equation~\ref{#1}}

\def\1{\bm{1}}

\DeclareMathAlphabet{\mathsfit}{\encodingdefault}{\sfdefault}{m}{sl}
\SetMathAlphabet{\mathsfit}{bold}{\encodingdefault}{\sfdefault}{bx}{n}

\usepackage{hyperref}
\usepackage{url}
\usepackage{algorithm,algorithmic}
\usepackage{subcaption}
\usepackage{graphicx}
\usepackage{amsthm,amsmath}
\usepackage{thm-restate}
\usepackage{booktabs,multirow,xcolor}
\usepackage{enumitem}
\usepackage{pifont}
\usepackage{wrapfig}
\usepackage{cleveref}
\usepackage{makecell}
\usepackage{dsfont}
\usepackage[dvipsnames]{xcolor}  
\hypersetup{
    colorlinks=true, 
    linkcolor=Blue,
    filecolor=Blue,      
    urlcolor=Blue,
    citecolor=Blue 
    }

\newcommand{\cmark}{\textcolor{green!55!black}{\ding{51}}}
\newcommand{\xmark}{\textcolor{red!75!black}{\ding{55}}}

\title{
A Foundation Model for Policy In-Context Longitudinal Causal Effect Estimation}

\author{Wenxin Chen\thanks{This study was done during a summer internship at Microsoft Research}\\
Cornell University \\
New York City, NY\\
\And
Kyra Gan \\
Cornell University \\
New York City, NY \\
\And
Mingyu Lu \\
University of Washington \\
Seattle, WA \\
\And
Zilin Jing \\
Columbia University \\
New York City, NY \\
\And
Fei Wang \\
Cornell University \\
New York City, NY \\
\And
Hoifung Poon \\
Recursion Pharmaceuticals, Inc. \\
Salt Lake City, UT \\
\And
Tristan Naumann \& Peniel Argaw \\
Microsoft Research \\
Redmond, WA
}

\newcommand{\methodname}{PICLE-PFN}

\newcommand{\bpi}{\boldsymbol{\pi}}

\theoremstyle{plain}

\theoremstyle{definition}

\theoremstyle{remark}

\iclrfinalcopy 
\begin{document}

\maketitle
\lhead{Preprint} 

\begin{abstract}
Longitudinal causal effect estimation is complicated by time-varying confounding and treatment-confounder feedback loops. Existing methods generally require data-specific nuisance-model fitting, and many are tailored to prespecified treatment sequences or prediction horizons, leading to time-consuming computation at deployment. We introduce \emph{Policy In-Context Longitudinal causal Effect estimation with Prior-data Fitted Networks} (\textbf{\methodname{}}), the first foundation model for longitudinal causal effect estimation that supports \emph{dynamic} treatment policies. \methodname{} represents target policies through in-context history-decision demonstrations, providing a unified interface for a broad class of treatment policies. Pretrained over synthetic data sampled from longitudinal structural causal models, \methodname{} amortizes causal effect estimation across datasets, policies, and horizons. Given a new observational dataset and target policy, it directly predicts conditional average potential outcomes (CAPOs) in one forward pass without task-specific training or hyperparameter tuning.
Across synthetic and semi-synthetic longitudinal benchmarks, \methodname{} achieves competitive RMSE and substantially improved uncertainty calibration with a fraction of deployment runtime, compared to state-of-the-art conventional estimators trained and tuned separately for each dataset.

\end{abstract}

\section{Introduction}
\vspace{-5pt}

Clinical care often involves a sequence of treatment decisions that evolve with a patient’s changing condition~\citep{murphy2003optimal, chakraborty2013statistical}. Estimating the effects of these time-varying treatments is therefore essential for guiding effective care. Doing so from observational data is particularly challenging: covariates can simultaneously be consequences of earlier treatments and determinants of later treatment decisions and outcomes, creating treatment–confounder feedback that complicates causal effect estimation~\citep{robins1986new, hernan2010causal}.

Existing methods decompose the 
longitudinal
counterfactual target into a sequence of local nuisance regressions, which are fitted to each target dataset and combined through a causal estimation procedure~\citep{robins1986new,robins2000marginal,bang2005doubly}. 
This decomposition introduces two practical challenges. First, estimation errors can propagate and accumulate over the intervention horizon~\citep{chen2026deep}. 
Second, existing methods face a trade-off between flexibility across causal queries and estimation variance. Some estimators fit nuisance models that can be reused across many policies and horizons, but this flexibility can come at the cost of high variance, e.g., through products of inverse propensities or by modeling the full longitudinal data-generating process~\citep{li2021g,frauen2025model}. Other estimators reduce variance by fitting only the outcome regressions needed for a particular policy and horizon, but changing the query then requires retraining and potentially new hyperparameter tuning~\citep{hess2026igc,shirakawa2024longitudinal}.
Moreover, many existing methods~\citep{lim2018forecasting,bica2020estimating,melnychuk2022causal,seedat2022continuous}
are designed for static treatment sequences, which
apply prescribed fixed treatment decisions throughout the trajectory without considering time-varying covariates.
Yet clinically relevant counterfactual questions may involve treatment regimes that adapt decisions to time-varying patient characteristics.
There is a need for an estimator that supports both static and dynamic policies across different horizons, without retraining for each new target.

To address these challenges, we explore an alternative approach based on in-context learning (ICL). 
Recent tabular foundation models demonstrate that pretraining on synthetic tasks can produce models that adapt in context to new datasets and prediction problems without parameter updates
~\citep{hollmann2023tabpfn, hollmann2025accurate, qu2025tabicl, qu2026tabiclv}. 
This paradigm offers two opportunities for longitudinal causal estimation. 
First, in synthetic data, we can intervene on a target policy and directly generate its corresponding multi-step potential outcome. This provides ground-truth supervision for the final causal estimand, avoiding the need to learn and compose a sequence of local nuisance models.
Second, diverse pretraining priors can support adaptation to different datasets and causal queries. 
The remaining gap is representing target policies, including dynamic ones, as context for ICL,
and directly learning conditional average potential outcome (CAPO) across different horizons.

To meet these gaps, we introduce \methodname{}, the first foundation model for longitudinal CAPO estimation supporting both static and dynamic policies. Our contributions are three-fold: \textbf{(i) Policy representation and ICL formulation.} We show that a target policy's decisions on factual histories across the \emph{entire} observed dataset preserve the policy information needed for CAPO identification. This motivates an ICL formulation that combines observational data with policy demonstrations, providing a unified interface for static and dynamic policies over multi-step intervention windows. \textbf{(ii) Architecture and amortized estimation.} We design a novel transformer based on Prior-data Fitted Network (PFN). Direct pretraining on synthetic multi-step CAPO targets avoids full generative modeling and explicit nuisance-model composition at deployment. Pretraining across diverse datasets, policies, and horizons enables flexible causal queries in a single forward pass without task-specific optimization. \textbf{(iii) Empirical performance.} 
\methodname{} achieves the lowest RMSE in six of seven benchmark settings, with up to 12.2\% lower RMSE than the strongest baseline, while being about $ 100\times$ to $ 5000\times$ faster at deployment than a comparably accurate method. It also provides substantially better calibrated uncertainty.
Code is available at \url{https://anonymous.4open.science/r/PICLE-02FE/}.

\section{Related Work}
\vspace{-5pt}

\begin{table*}[t]
\centering
\caption{\textbf{Flexibility of existing methods:}
RMSN \citep{lim2018forecasting};
CRN \citep{bica2020estimating};
TE-CDE \citep{seedat2022continuous};
CT \citep{melnychuk2022causal};
G-Net \citep{li2021g};
IGC-Net \citep{hess2026igc};
DeepACE \citep{frauen2023estimating};
DeepLTMLE \citep{shirakawa2024longitudinal};
D$^3$-Net \citep{chen2026deep}.
\cmark\, indicates that the capability is supported by the method's
released implementation. 
}
\label{tab:method_flexibility}

\resizebox{\textwidth}{!}{%
\begin{tabular}{lcccccccccc}
\toprule
\textbf{Capability}
    & RMSN
    & CRN
    & TE-CDE
    & CT
    & G-Net
    & IGC-Net
    & DeepACE
    & DeepLTMLE
    & D$^3$-Net
    & \textbf{\methodname{}} \\
\midrule

Dynamic treatment policy
    & \xmark & \xmark & \xmark & \xmark & \xmark
    & \cmark & \xmark & \cmark & \cmark & \cmark \\

Multidimensional treatment
    & \cmark & \cmark & \cmark & \cmark & \cmark
    & \cmark & \xmark & \xmark & \xmark & \cmark \\

No retraining across horizons/policies
    & \cmark & \cmark & \cmark & \cmark & \cmark
    & \xmark & \xmark & \xmark & \xmark & \cmark \\

\bottomrule
\end{tabular}%
}
\vspace{-1em}
\end{table*}

\paragraph{Tabular Foundation Model for Time Series and Causal Inference}
Recent work adopts the PFN approach for time-series forecasting. 
They either repurpose a pretrained non-longitudinal tabular model by treating time as additional features~\citep{hoo2026from} or explicit model temporal structure~\citep{dooley2023forecastpfn,taga2025timepfn,potapczynski2026timeaware}. 
Our model follows the latter approach, building an architecture to respect the temporal ordering of trajectories.
A separate line of work applies PFN to causal inference~\citep{balazadeh2025causalpfn,robertson2026pfn,ma2026foundation}. Yet they only focus on the static, single-treatment-assignment setting. 
In contrast, we consider the longitudinal setting with sequential treatment assignments and time-varying confounders\footnote{Concurrent work by \cite{zare2026causal} approach this longitudinal setting using a one-step mean rollout that does not generally recover CAPO, and only supports static treatment sequences; see Appendix~\ref{appx:causallongpfn}.}.  

\paragraph{Longitudinal Causal Effect Estimation.} Longitudinal causal estimators address time-varying confounding through modeling treatment, outcome, or both. Treatment-based methods use inverse probability weighting or balanced representations \citep{robins2000marginal,bica2020estimating,seedat2022continuous,lim2018forecasting,melnychuk2022causal}, but can suffer from unstable weights under limited overlap \citep{cole2008constructing,petersen2012diagnosing}. Outcome-based methods use the G-formula, either by simulating covariate trajectories~\citep{li2021g} or recursively regressing outcomes backward in time \citep{mcgrath2020gformula,frauen2023estimating,oprescu2025gst,hess2026igc}. Doubly robust methods combine treatment and outcome models to reduce sensitivity to misspecification \citep{van2011targeted,frauen2025model,shirakawa2024longitudinal,chen2026deep,chen2026smooth}. 

Despite these advances, existing methods show limited flexibility to handle various causal queries. 
Many focus on prespecified treatment sequences that do not adapt to patients' evolving status. Some require retraining or potentially retuning when changing the target policy or horizon. 
Recent DR-based methods also commonly assume a single-dimension treatment, with extensions to multidimensional treatment spaces requiring nontrivial changes to their models and objectives. 
In contrast, \methodname{}
provides all these capabilities
within a single pretrained model
(Table~\ref{tab:method_flexibility}).

\section{Problem Setup and In-Context Formulation}
\vspace{-5pt}

We consider $n$ independent and identically distributed (i.i.d.) trajectories observed over discrete time steps $t \in \{1, \dots, T\}$. For each trajectory, we observe time-invariant baseline covariates $B \in \mathbb{R}^{d_B}$, time-varying covariates $L_t \in \mathbb{R}^{d_L}$, and a $d_A$-dimensional binary treatments $A_t \in \{0, 1\}^{d_A}$. Let $Y_{t+1} \in \mathbb{R}$ denote the outcome measured following treatment assignment $A_t$. 
An observed trajectory $O
$ is thus represented as:
$
O = (B, L_1, A_1, Y_2, L_2, A_2, Y_3, \dots, L_T, A_T, Y_{T+1})\sim P_0,
$
where $P_0$ is the unknown true data-generating distribution.

For any sequence of random variables $X_t$, let $X_{s:e} = (X_s, X_{s+1}, \dots, X_e)$ denote the subsequence from time $s$ to $e$, with $X_{s:e}$ empty when $s>e$. 
The observed history immediately preceding treatment assignment at time $t$ is denoted as
$
H_t = (B, L_{1:t}, A_{1:t-1}, Y_{2:t}),
$ 
and the full trajectory up to time $t$ as $
O_{t} = (B, L_{1:t}, A_{1:t}, Y_{2:t+1}).
$


Under $P_0$, treatments follow an unknown behavior policy $\pi^*=(\pi_1^*,\ldots,\pi_T^*)$, where $\pi_t^*(a\mid h_t):=\Pr(A_t=a\mid H_t=h_t)$ is the probability of assigning treatment $a$ given history $h_t$.

\paragraph{Target Policies and Target Estimand.}
We consider deterministic target policies $\pi_{s:e} = (\pi_s, \dots, \pi_e)$, where $\pi_t:\mathcal H_t\to\{0,1\}^{d_A}$.
This includes both static treatment sequences and
history-dependent dynamic policies.
Denote $a^{\pi_{s:e}}_{s:e}$ as the treatment sequences that would have been assigned if we apply target policy $\pi_{s:e}$ starting from $H_s$, and denote $L_{e+1}^{\pi_{s:e}}$ and $Y_{e+1}^{\pi_{s:e}}$ as the potential covariates and outcome at time $e+1$, respectively, under $\pi_{s:e}$.

Given a history $h_s\in\mathcal H_s$, our target estimand is \textbf{conditional average potential outcome (CAPO)},
\[
\psi_0^{\pi_{s:e}}(h_s)
:=
\mathbb E\!\left[
Y_{e+1}^{\pi_{s:e}}\mid H_s=h_s
\right].
\]
Intuitively, $\psi_0^{\pi_{s:e}}(h_s)$ represents the true expected future outcome at time $e+1$ conditional on observed history $h_s$
and treatment
policy $\pi_{s:e}$ over the horizon $s$ to $e$. 

\paragraph{Identification Assumptions.}
We adopt the following standard assumptions~\citep{shirakawa2024longitudinal,wen2021parametric}:

\begin{enumerate}[label=(\alph*), itemsep=0pt, topsep=2pt, partopsep=0pt, parsep=0pt, leftmargin=20pt]
    \item Consistency. Observed covariates/outcome under realized treatments equal the potential covariates/outcome under that same treatments: $A_{s:e} = a^{\pi_{s:e}}_{s:e} \implies L_{e+1}^{\pi_{s:e}} = L_{e+1},\, Y_{e+1}^{\pi_{s:e}} = Y_{e+1}$. 
    
    \item Sequential Ignorability. For every $t\in\{s,\dots,e\}$
        and target policy sequence $\pi_{t:e}$ under consideration,
        $
        (Y^{\pi_{t:t}}_{t+1},\dots,Y^{\pi_{t:e}}_{e+1})
        \perp A_t \mid H_t.
        $
    \item Positivity. For histories $h_t \in \mathcal{H}_t$ and $a = \pi_t(h_t)$:
    $
    \pi_t^*(a \mid h_t) > 0.
    $
\end{enumerate}

With these assumptions, a target policy's CAPO is identified. We next show that a policy can be represented by its decisions on observed histories without losing information needed for identification.


\subsection{In-context Formulation of Longitudinal Causal Effect Estimation}

Target policies can take 
various forms, ranging from static treatment sequences
to complex, black-box decision rules. 
Parameterizing their underlying mechanisms as model inputs would necessitate a distinct representation for each policy family. To circumvent this, we represent a policy through its decisions on the dataset's factual histories rather than its underlying rules or parameters. We refer to the set of history-decision pairs as a \textbf{target policy demonstration}, which serves as a universal interface for any \emph{deterministic policy}. 
We first show that, under the standard assumptions, the target policy demonstrations
suffice to identify the CAPO (proof in Appendix~\ref{proof:policy_demo_sufficiency}).

\begin{restatable}[Sufficiency of Factual History Policy Demonstrations]{proposition}{PropPolicyDemoSufficiency}
\label{prop:policy_demo_sufficiency}

Suppose consistency, sequential ignorability, and positivity hold. For every $t\in\{s,\dots,e\}$, construct a population demonstration of a deterministic policy according to
$$
H_t\sim P_0,
\qquad
A_t^{\pi_t}\sim\pi_t(H_t).
$$
Then the collection of population policy demonstration distributions of
$(H_t,A_t^{\pi_t})$, together with $P_0$, identifies the corresponding target CAPO,
$
\psi_0^{\pi_{s:e}}(h_s),
$
for $P_{0}$-almost every $h_s$.
\end{restatable}




Proposition~\ref{prop:policy_demo_sufficiency} establishes sufficiency at the population level. In practice, we approximate the population demonstrations with the empirical dataset by applying the policy to the observed histories:
$$
D_{n}^{\pi_{s:e}}
=
\left\{
(H_{i,t},A_{i,t}^{\pi_t})
\right\}_{i=1:n,\ t=s:e}
\quad
\text{where } A_{i,t}^{\pi_t}\sim\pi_t(H_{i,t}).
$$

We note that the policy information resides in the conditional relationship between histories and treatment decisions \emph{across the entire observed dataset}. Any single history-treatment pair generally does not specify this relationship. An estimator must therefore aggregate information across demonstrations to estimate the CAPO. This naturally motivates an ICL formulation: the observed dataset plus policy demonstrations serves as context for inferring the relevant policy behavior.



Besides the policy demonstrations, an estimator also needs to learn the data distribution. Under the ICL formulation, we can provide the observed trajectories $D_{n,e}^{\mathrm{obs}}=\{O_{i,e}\}_{i=1}^n$ as additional context. 
Because each observed trajectory $O_{i,e}$ already contains the patient's factual history, the corresponding policy demonstration for that patient is the sequence of target actions $(A_{i,t}^{\pi_t})_{t=s}^{e}$, prescribed for that specific history. This allows us to combine the two contexts efficiently without duplicating the history data, yielding the joint input: $ 
\left\{
\left(
O_{i,e},
(A_{i,t}^{\pi_t})_{t=s}^{e}
\right)
\right\}_{i=1}^{n}.
$ We feed this combined context, along with a new query history $h_s$, into an amortized estimator $\mathcal F_\theta$:
$$
\hat{\psi}_{n}^{\pi_{s:e}}(h_s)
:=
\mathcal F_\theta
\left(
D_{n,e}^{\mathrm{obs}},
D_n^{\pi_{s:e}},
h_s
\right) = 
\mathcal F_\theta
\left(\{
(
O_{i,e},
(A_{i,t}^{\pi_t})_{t=s}^{e}
)
\}_{i=1}^{n}, h_s
\right),
$$
where $\mathcal F_\theta$ is a shared parameterized predictor trained to estimate $\psi_0^{\pi_{s:e}}(h_s)$ across datasets and policies. Once trained, it estimates CAPO for new datasets, target policies, and query histories through ICL, without dataset- or policy-specific parameter updates.

\noindent \textbf{Supervision for Multi-step Counterfactual Targets.} The formulation above specifies the information available at inference time, but not how to supervise the estimator. Although the target CAPO is identified, it is not directly observed as a multi-step training label. For one-step outcome regression, the observed outcome provides a direct supervision signal: regressing $Y_{t+1}$ on $(H_t,A_t)$ estimates the conditional mean needed to evaluate a target policy. Over multiple steps, however, the observed endpoint reflects the subsequent treatments assigned under the behavior policy, rather than the target policy. It therefore does not generally provide a valid regression label for the multi-step CAPO under a target policy. Conventional G-computation methods bridge this gap by fitting local nuisance models and composing their predictions across time~\citep{li2021g,shirakawa2024longitudinal}.

Our ICL estimator bypasses explicit nuisance-model fitting and directly maps observational and policy contexts to CAPO, requiring a different source of supervision.
We therefore pretrain on synthetic datasets, from which we can generate both the observed data and corresponding CAPO targets through simulation under the target policy. 
The combination of ICL and synthetic supervision makes the PFN paradigm a natural fit. 

\section{Method}
\begin{algorithm}[t]
\caption{Longitudinal Causal PFN Pre-Training}
\label{alg:longitudinal_causal_pfn_prior}
\begin{algorithmic}[1]
\STATE \textbf{Input:} Prior over longitudinal SCMs $\Pi$ and target policies $\Pi_{\mathrm{policy}}$

\STATE Sample a longitudinal SCM
$
\mathcal{M} \sim \Pi,
$
which includes the behavior policy $\bpi^*$

\STATE Sample a target counterfactual policy:
$
\bpi \sim \Pi_{\mathrm{policy}}.
$

\FOR{$i=1,\dots,n$}
    \STATE Generate an observed trajectory, $O_i$, of length $T$ under the sampled SCM and behavior policy.

    \FOR{$(s,e)\in\mathcal K$}
        \STATE Obtain target policy demonstration by evaluating $\pi_t$ on the observed histories at each time step:
        $A_{i,t}^{\pi_t} = \pi_t(H_{i,t}), \;\text{for } t=s,\dots,e.$
        \STATE Simulate true counterfactual outcome: Branch a counterfactual rollout from $H_{i,s}$ by applying $\pi^{s:e}$ on the evolving counterfactual history, which yields the label $Y^{\pi_{s:e}}_{i,e+1}$.
    \ENDFOR

\ENDFOR

\STATE Partition the trajectory indices into disjoint context and
query sets



\STATE Minimize the negative log-likelihood over intervention windows and query trajectories. 
(Eq.\ref{eq:objective})

\end{algorithmic}
\end{algorithm}

We introduce \methodname{}, a PFN-style model pretrained on tasks sampled from a prior over longitudinal structural causal models and target policies. The prior (\S~\ref{sec:prior}) supplies diverse datasets, causal queries, and counterfactual labels; the training objective (\S~\ref{sec:objective}) enables direct CAPO estimation without fully generative modeling (summarized in Algorithm~\ref{alg:longitudinal_causal_pfn_prior}). The architecture (\S~\ref{sec:architecture}) integrates observational data and policy demonstrations to jointly predict across intervention windows.

\subsection{Prior Sampling}
\label{sec:prior}
We generate synthetic training tasks by sampling longitudinal SCMs and target policies. Each SCM varies in its dimensions, graph skeleton, and transition mechanisms, capturing nonlinear dynamics, heterogeneous treatment effects, and diverse behavioural policies. We construct the graph by sampling parent variables and their time lags while enforcing acyclicity. Target policies include both static treatment sequences and dynamic rules. For each task, we generate factual trajectories and counterfactual rollouts under the target policy; the latter provide supervision labels whose conditional mean is the CAPO. Appendix~\ref{appx:prior} details the sampling procedure.

\subsection{Training Objective}
\label{sec:objective}

We train \methodname{} to estimate CAPOs through the predictive means of counterfactual outcome distributions, jointly across query trajectories and intervention windows. For start time $s$ and end time $e$, let $\tau=e-s+1$ denote the horizon. The set of admissible windows is
\[
\mathcal K=\{(s,e):1\le s\le e\le T,\ e-s+1\le\tau_{\max}\}.
\]
Let $[n]=\{1,\ldots,n\}$ index the trajectories in a training task.
We partition these indices into disjoint context and query sets
$\mathcal I_{\mathrm{ctx}}$ and $\mathcal I_{\mathrm{qry}}$, with
$\mathcal I_{\mathrm{ctx}}\cup\mathcal I_{\mathrm{qry}}=[n]$, $\mathcal I_{\mathrm{ctx}}\cap\mathcal I_{\mathrm{qry}}=\emptyset$.
For each $(s,e)\in\mathcal K$, the inputs are
$$
D_{s:e}^{\mathrm{ctx}}
=
\left\{
\left(O_{i,e},\{A_{i,t}^{\pi_t}\}_{t=s}^{e}\right):
i\in\mathcal I_{\mathrm{ctx}}
\right\},
\qquad
D_{s:e}^{\mathrm{qry}}
=
\left\{
(H_{i,s},A_{i,s}^{\pi_s}):
i\in\mathcal I_{\mathrm{qry}}
\right\}.
$$
The context provides factual trajectories through the endpoint and policy demonstrations throughout the intervention window. Each query provides only its pre-intervention history and initial policy demonstration; information about subsequent policies is supplied by the context demonstrations.

For each query and window, we branch a counterfactual rollout from $H_{i,s}$ and apply the target policy through step $e$, obtaining a realized potential outcome $Y_{i,e+1}^{\pi_{s:e}}$. 
For a fixed SCM and target policy, this outcome has conditional expectation equal to the CAPO given $H_{i,s}$. We therefore use single-draw outcomes as a computationally efficient surrogate for direct CAPO supervision, avoiding the repeated rollouts required to approximate conditional means explicitly. We minimize the outcome negative log-likelihood:
\begin{equation}\label{eq:objective}
    \mathcal L(\theta)
=
\mathbb E\!\left[
-\frac{1}{|\mathcal I_{\mathrm{qry}}|\,|\mathcal K|}
\sum_{i\in\mathcal I_{\mathrm{qry}}}
\sum_{(s,e)\in\mathcal K}
\log q_\theta\!\left(
Y_{i,e+1}^{\pi_{s:e}}
\mid H_{i,s},A_{i,s}^{\pi_s},D_{s:e}^{\mathrm{ctx}}
\right)
\right],
\end{equation}
where the expectation is over sampled tasks. At inference, the predictive mean provides the CAPO estimate. Because the model is trained on realized counterfactual outcomes, its predictive spread also includes outcome variability.

\begin{figure}
    \centering
    \includegraphics[width=0.72\linewidth]{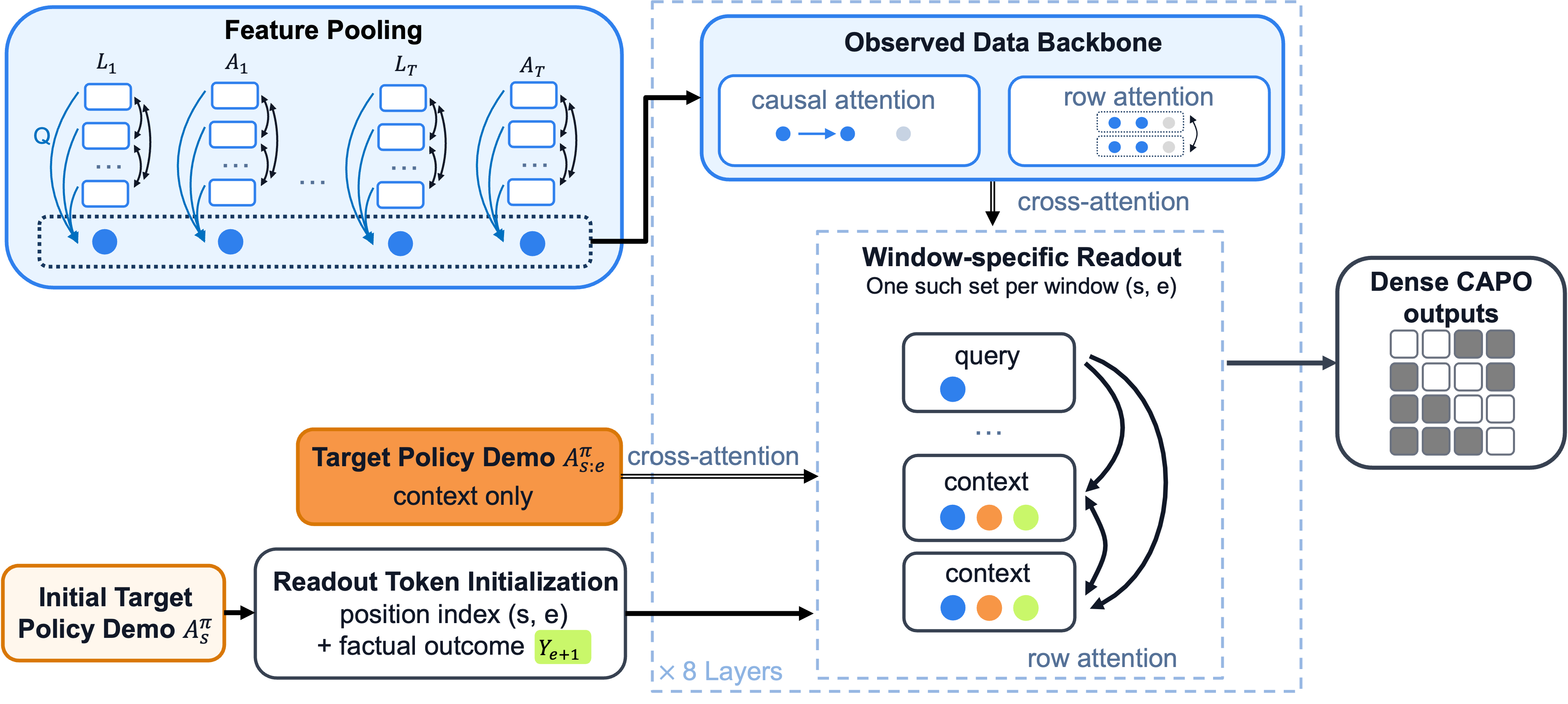}
    \caption{\textbf{PICLE-PFN Architecture}. Feature pooling encodes longitudinal covariates and treatments into embeddings processed by a shared backbone with causal temporal attention and row-attention. Window-specific readouts combine window indices, initial policy demonstrations, and factual outcomes (masked for queries). They cross-attend to trajectory-specific backbone representations under window-specific masks: query readouts access only pre-intervention history, while context readouts access factual data through $e$ and additionally attend to target policy demonstrations. Row-attention uses only context keys and values and operates separately within each readout window. Query readouts yield CAPO predictions across admissible windows in one forward pass.}
    \label{fig:architecture}
    \vspace{-1em}
\end{figure}

\subsection{Model Architecture}
\label{sec:architecture}
We parametrize $\mathcal{F}_{\theta}$ using a PFN-style transformer~\citep{hollmann2025accurate,qu2025tabicl} with a shared \emph{observed-data backbone} and \emph{window-specific readout tokens}. Each trajectory has one readout token per intervention window (Figure~\ref{fig:architecture}). Each transformer layer first updates the observed-data representations through causal attention across time steps and row attention across trajectories, then updates the readout tokens. Each readout cross-attends to its own trajectory under a window-specific mask: for window $(s,e)$, query readouts access only the factual history $H_s$, while context readouts access the factual trajectory through the target endpoint, including $Y_{e+1}$. The context readouts additionally incorporate target-policy demonstrations over $s{:}e$. Readout tokens then attend across trajectories, using only context readouts as keys and values, to enable in-context learning. Finally, the query readouts predict CAPO. Readout tokens share model parameters across windows, producing predictions for all admissible windows in a single forward pass. Full details appear in Appendix~\ref{appx:architecture}.

\section{Experiments}
\vspace{-5pt}
We first describe the datasets, baselines, and evaluation protocol (\S~\ref{sec:setup}). We then report CAPO estimation performance, a sanity check on real-world data prediction, and additional analyses on uncertainty calibration and ablation studies (\S~\ref{sec:results}). 
\vspace{-5pt}
\subsection{Experimental Setup}
\vspace{-5pt}
\label{sec:setup}

\begin{wraptable}{r}{0.38\textwidth}
\centering
\small
\setlength{\tabcolsep}{3pt}
\caption{Dataset characteristics. }
\label{tab:benchmark_settings}
\resizebox{\linewidth}{!}{%
\begin{tabular}{@{}lcc@{}}
\toprule
Dataset & $(d_A,d_L,d_B)$ & Policy type\\
\midrule
\texttt{cancer}          & $(2,1,1)$  & Static \\
\texttt{mimic\_y1}       & $(3,27,3)$ & Static \\
\texttt{mimic\_y2}       & $(3,27,3)$ & Static \\
\texttt{mimic\_1t}       & $(1,27,3)$ & Static \\
\texttt{mimic\_1t\_hdp0.5} & $(1,27,3)$ & Dynamic \\
\texttt{mimic\_1t\_hdp1} & $(1,27,3)$ & Dynamic \\
\texttt{mimic\_1t\_hdp2} & $(1,27,3)$ & Dynamic \\
\bottomrule
\multicolumn{3}{c}{\tiny Larger HDP values indicate greater deviation from the behavioral policy.}
\end{tabular}}
\end{wraptable}

\paragraph{Datasets \& Metrics.}  We use seven dataset settings, where the range of settings spans varying dimensions of treatments ($d_A$) and covariates ($d_L,d_B$), as well as static and dynamic treatment policies (\Cref{tab:benchmark_settings}). 
In the three \texttt{mimic\_1t\_hdp*} settings, their target policy functions have the same functional form as the behavioral policy, but have a scalar perturbation term that controls the deviation from the behavioral policy. The detailed data-generating process can be found in Appendix~\ref{appx:benchmark_data}.
Each setting synthesizes 20 datasets with 1000 trajectories in the training (context) set, and 1000 in the test (query) set. We estimate the test set CAPO on five different horizons, $\tau \in \{2,4,6,8,10\}$. We report RMSE, normalized by the standard deviation of observed outcomes.
We additionally evaluate \methodname{} on a real-world prediction task as a sanity check, described in \S~\ref{sec:results} and Appendix~\ref{appx:mimic_real}.

\noindent \textbf{Baselines.} We compare against nine baselines.
Given their flexibility to handle different settings, not all of them are evaluated in all baselines: (1) \textbf{DeepACE}~\citep{frauen2023estimating}, \textbf{DeepLTMLE}~\citep{shirakawa2024longitudinal}, and \textbf{D3-Net}~\citep{chen2026deep} can only handle single-dimensional treatment, so they are excluded from \texttt{cancer}, \texttt{mimic\_y1}, and \texttt{mimic\_y2}; (2) \textbf{CRN}~\citep{bica2020estimating}, \textbf{RMSN}~\citep{lim2018forecasting}, \textbf{TE-CDE}~\citep{seedat2022continuous}, \textbf{Causal Transformer (CT)}~\citep{melnychuk2022causal}, and \textbf{G-Net}~\citep{li2021g} don't handle dynamic treatment policies, so they are excluded from \texttt{mimic\_hdp0.5/1/2}; (3) \textbf{IGC-Net} is evaluated on all settings.
For every baseline, we tune its hyperparameters (HPs) with Optuna for
2 hours.The HP search spaces are listed in Table~\ref{tab:hp_grid}.
Further details can be found in Appendix~\ref{appx:baseline_hparams}.
\begin{table*}[t]
\centering
\caption{\textbf{Longitudinal CAPO benchmark across datasets.} Each cell reports the mean$\pm$SD normalized RMSE over 20 runs, with RMSE averaged across horizons within each repeat. \textbf{Bold} and \underline{underlined} values denote the best and second-best methods in each column, respectively.}
\label{tab:rmse_benchmark}
\resizebox{0.9\textwidth}{!}{%
\begin{tabular}{@{}lccccccc@{}}
\toprule
\multirow{2}{*}{Method} & \multirow{2}{*}{Cancer} & \multicolumn{6}{c}{MIMIC} \\
\cmidrule(lr){3-8}
 &  & Y1 & Y2 & 1T & 1T-HDP0.5 & 1T-HDP1 & 1T-HDP2 \\
\midrule
\textbf{PICLE-PFN} & \textbf{0.190$\pm$0.03} & 0.958$\pm$0.50 & \textbf{0.603$\pm$0.23} & \textbf{0.583$\pm$0.20} & \textbf{0.527$\pm$0.13} & \textbf{0.534$\pm$0.13} & \textbf{0.555$\pm$0.15} \\
IGC-Net & 0.297$\pm$0.02 & \underline{0.901$\pm$0.49} & \underline{0.604$\pm$0.18} & 0.699$\pm$0.11 & \underline{0.546$\pm$0.13} & \underline{0.580$\pm$0.13} & \underline{0.622$\pm$0.17} \\
CT & \underline{0.217$\pm$0.06} & \textbf{0.851$\pm$0.29} & 0.886$\pm$0.82 & 0.660$\pm$0.40 & --- & --- & --- \\
RMSN & 0.357$\pm$0.04 & 1.242$\pm$0.70 & 0.830$\pm$0.13 & 0.786$\pm$0.21 & --- & --- & --- \\
G-Net & 0.661$\pm$0.04 & 1.333$\pm$0.86 & 0.984$\pm$0.14 & 1.127$\pm$0.20 & --- & --- & --- \\
CRN & 0.325$\pm$0.05 & 1.051$\pm$0.34 & 0.845$\pm$0.37 & 0.821$\pm$0.35 & --- & --- & --- \\
TE-CDE & 1.122$\pm$0.08 & 1.305$\pm$0.38 & 1.283$\pm$0.37 & 1.146$\pm$0.20 & --- & --- & --- \\
D3 & --- & --- & --- & \underline{0.620$\pm$0.11} & 0.567$\pm$0.14 & 0.587$\pm$0.15 & 0.633$\pm$0.17 \\
DLTMLE & --- & --- & --- & 0.852$\pm$0.21 & 0.662$\pm$0.16 & 0.685$\pm$0.17 & 0.734$\pm$0.20 \\
DeepACE & --- & --- & --- & 0.682$\pm$0.21 & --- & --- & --- \\
\midrule
$\Delta$ vs.\ best baseline & -12.2\% & +12.6\% & -0.3\% & -6.1\% & -3.6\% & -7.9\% & -10.8\% \\
\bottomrule
\end{tabular}}
\end{table*}

\subsection{Results}
\label{sec:results}
\noindent \textbf{Synthetic/Semi-synthetic Data} 
Table~\ref{tab:rmse_benchmark} shows that \methodname{} achieves the lowest horizon-averaged RMSE in six of seven settings, reducing RMSE over the best baseline by 12.2\% on \texttt{cancer} and by up to 10.8\% on the remaining MIMIC benchmarks. The exception is \texttt{mimic\_y1}, where it ranks third, with an RMSE 12.6\% higher than that of the best baseline.
Figure~\ref{fig:rmse_top3} selects the top three models in Table~\ref{tab:rmse_benchmark} and visualizes their RMSE results for every horizon in each data setting, showing that its competitive performance holds across horizons (Full results across horizons can be found in Table~\ref{tab:full_benchmark_per_horizons}). 

Table~\ref{tab:runtime} shows measurements of deployment runtime on three representative settings: \methodname{} is about $100\times$ faster than CT and $5000\times$ faster than IGC-Net, the two strongest baselines with comparable accuracy in many benchmark settings. Figure~\ref{fig:rank_runtime_tradeoff} further compares the RMSE ranking and deployment runtime; results are aggregated across horizons (Appendix~\ref{appx:runtime_measurement}). 
Together, these results show that \methodname{} combines top-ranking accuracy across diverse settings with substantially faster runtime.

\begin{wraptable}{r}{0.55\textwidth}
    \centering
    \caption{\textbf{Factual Outcome Prediction on MIMIC.} Normalized RMSE of the factual diastolic blood pressure on the held-out test patients of 10 sampled cohorts. 
    }
    \label{tab:mimic_real}
    \resizebox{\linewidth}{!}{%
    \begin{tabular}{lcccc}
    \toprule
    Method & $\tau=2$ & $\tau=3$ & $\tau=4$ & $\tau=5$ \\
    \midrule
    \textbf{PICLE-PFN} & \textbf{0.730$\pm$0.030} & \textbf{0.737$\pm$0.032} & \textbf{0.754$\pm$0.026} & \textbf{0.767$\pm$0.037} \\
    IGC-Net & 0.736$\pm$0.040 & 0.740$\pm$0.038 & 0.762$\pm$0.031 & 0.779$\pm$0.028 \\
    \midrule
    PICLE-PFN wins & 7/10 & 5/10 & 8/10 & 8/10 \\
    \bottomrule
    \end{tabular}}
\end{wraptable}

\noindent \textbf{Real-world Data.} We further evaluate prediction under the behavioral policy on real-world MIMIC data, using vasopressor administration and mechanical ventilation as treatments and multi-step-ahead diastolic blood pressure as the outcome. To avoid introducing a separately fitted behavioral model, we use factual treatments in IGC-Net's recursive regression targets and as policy demonstrations for \methodname{}; implementation details appear in Appendix~\ref{appx:mimic_real}. IGC-Net is the only baseline that accommodates both multivariate treatments and dynamic policies as in this setting, so we restrict the comparison accordingly. Table~\ref{tab:mimic_real} shows that \methodname{} achieves slightly lower prediction error across all horizons, suggesting that it retains strong predictive performance on factual outcomes in real-world data. We view this experiment as a real-world sanity check, complementary to the synthetic and semi-synthetic benchmarks where counterfactual ground truth is available.

\begin{figure}[t]
    \centering
    \includegraphics[width=0.8\linewidth]{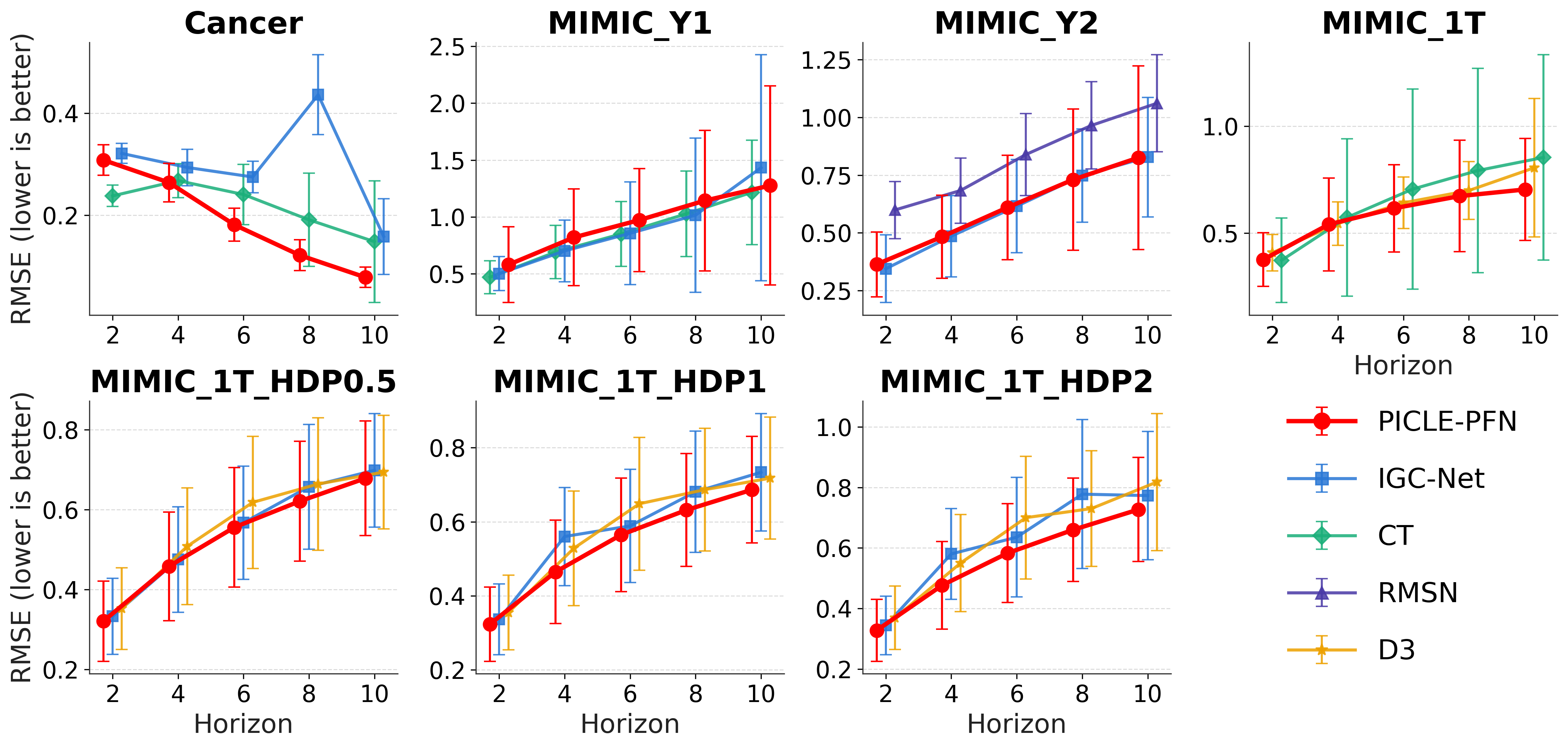}
    \caption{RMSE across horizons of the top-3 methods in each setting. Points are the normalized RMSE averaged over the 20 tasks, error bars are $\pm 1$ SD across tasks}
    \label{fig:rmse_top3}
\end{figure}

\begin{figure}[t]
    \centering
    \includegraphics[width=0.9\linewidth]{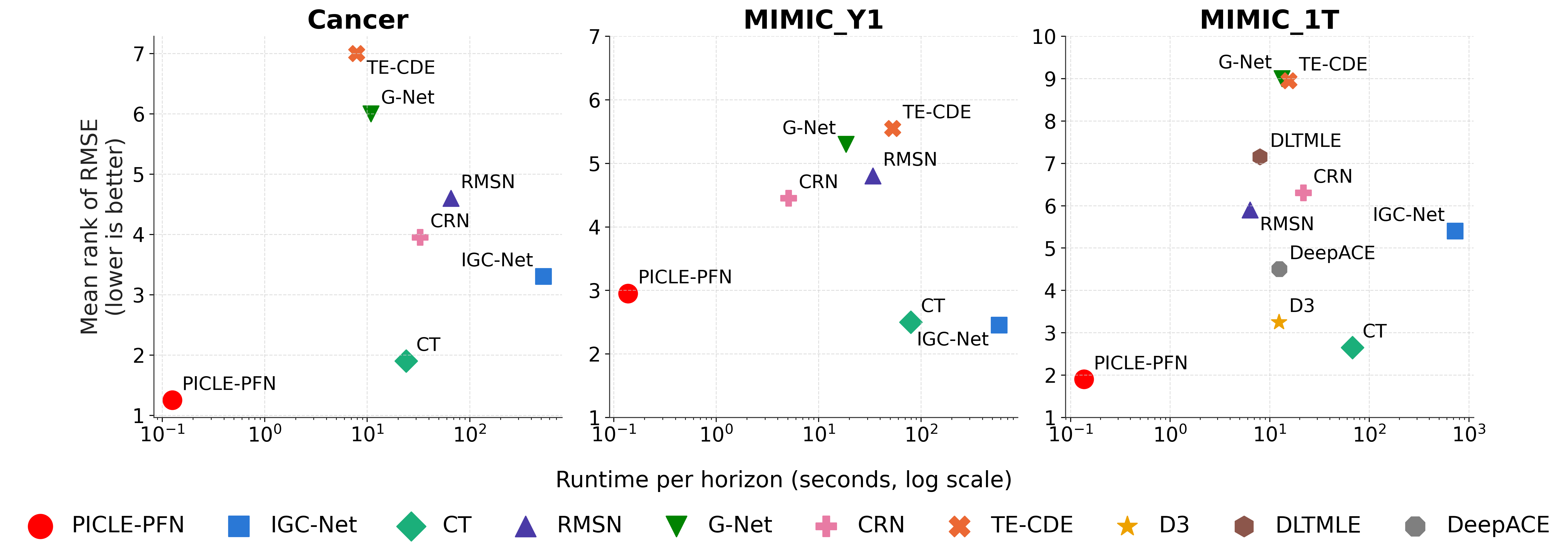}
    \caption{RMSE rank (y-axis) versus inference runtime (x-axis; log scale) across datasets.}
    \label{fig:rank_runtime_tradeoff}
    \vspace{-1em}
\end{figure}

\noindent \textbf{Uncertainty Calibration.}\;
We compare the empirical coverage of CAPO by the intervals produced by \methodname{} and the strongest baseline, IGC-Net. \methodname{} outputs a predictive distribution in a single forward pass, from which we extract equal-tailed $100(1-\alpha)\%$ intervals.\footnote{The single-draw supervision during pretraining introduces outcome variability, so these should not be interpreted as credible intervals for CAPO.} For IGC-Net, we construct intervals using MC dropout, following its original uncertainty-quantification approach. A query trajectory is covered if its true CAPO lies inside the interval; task-level coverage is the fraction of covered queries. We plot empirical CAPO coverage against the nominal interval level $1-\alpha$, using the diagonal as a reference. Evaluation details are in Appendix~\ref{appx:uncertainty_calibration}.

Figure~\ref{fig:calibration_curve} (left) compares uncertainty calibration on \texttt{mimic\_y1} and \texttt{mimic\_1t} through empirical CAPO coverage, averaged across horizons. \methodname{} tracks the nominal coverage levels substantially more closely than IGC-Net, whose MC-dropout uncertainty estimates are severely overconfident. Across horizons, \methodname{} maintains coverage near 0.9 at the nominal 95\% level (Figure~\ref{fig:calibration_curve}, middle), while its predictive uncertainty increases, as reflected by wider intervals (Figure~\ref{fig:calibration_curve}, right). This is the desired behavior: the model widens its intervals to reflect the growing uncertainty of longer horizons, and thereby maintains coverage. In contrast, IGC-Net's intervals do not widen with longer horizons, and even narrow in one setting, suggesting its substantial overconfidence across horizons. These results show that \methodname{} expresses predictive uncertainty more consistently with its observed CAPO estimation errors.

\begin{figure}[t]
    \centering
    \includegraphics[width=0.8\linewidth]{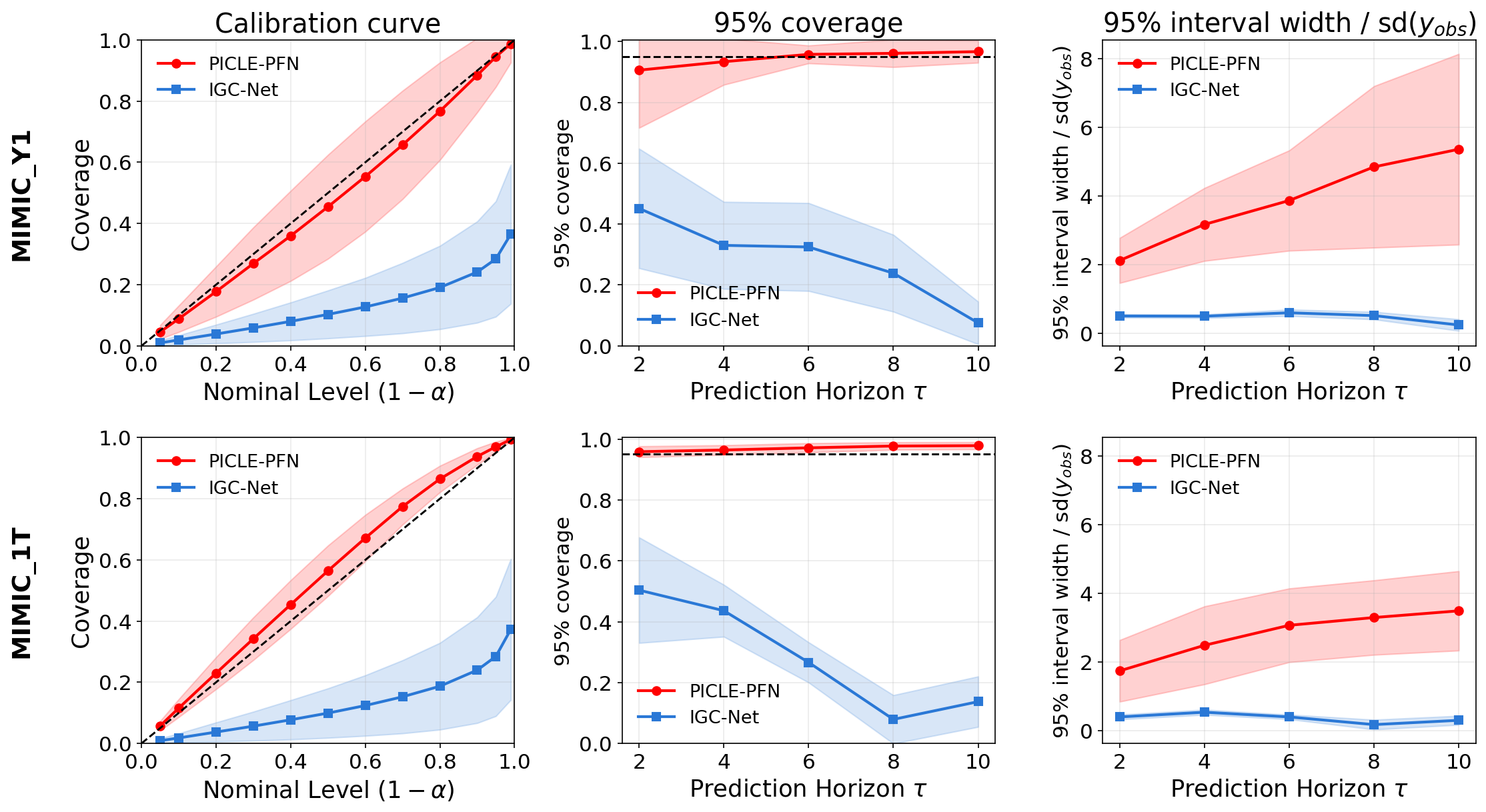}
    \caption{\textbf{Uncertainty Calibration Analysis.} Left: empirical coverage against its nominal level, averaged across horizons and 20 tasks; dashed diagonal = perfect calibration. Middle: coverage of the 95\% interval per horizon
(dashed line = nominal). Right: 95\% interval width per horizon.}
    \label{fig:calibration_curve}
\end{figure}

\begin{figure}[t]
    \centering
    \includegraphics[width=0.8\linewidth]{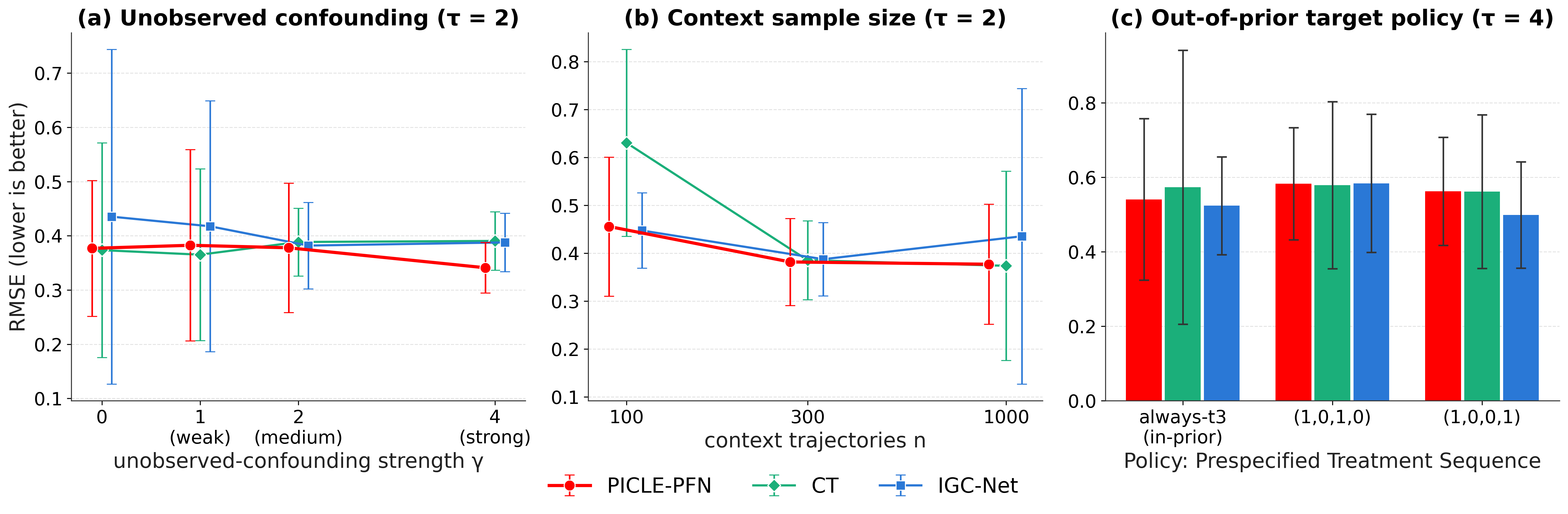}
    \caption{\textbf{Ablation studies.} (a) Unobserved confounding. (b) Context sample size. (c) Our-of-prior target policy: Static alternating treatments. Compared with the in-prior always-treat policy.}
    \label{fig:ablation}
    \vspace{-1em}
\end{figure}

\noindent \textbf{Unobserved Confounding.}\; 
We examine how \methodname{} and the two strongest baselines, CT and IGC-Net under unobserved confounder.
On \texttt{mimic\_1t}, we draw 
a confounder $U \sim \mathcal{N}(0,1)$ 
that is hidden from all methods and enters the data-generating process twice by shifting the patient's outcome 
and propensity
at every step. 
We vary the strength of unobserved confounding (details in Appendix~\ref{appx:unobserved_confounding}).
Figure~\ref{fig:ablation}(a) shows that \methodname{} and IGC-Net remain relatively stable across the tested range, whereas CT's performance deteriorates as confounding becomes stronger.

\noindent \textbf{Varying Context Sample Size.}\; 
Smaller samples increase finite-sample error and, for our method, provide fewer demonstrations of the target policy.
We test robustness under smaller context size, varying the context size 
on \texttt{mimic\_1t} (static) and \texttt{mimic\_1t\_hdp1} (dynamic policy). 
Baselines are re-tuned and re-trained at every size. Figure~\ref{fig:ablation}(b) and Table~\ref{tab:ablation_sample_size} show that RMSE increases for all methods as sample decreases, but \methodname{} remains more robust.

\noindent \textbf{Target Policies Outside the Prior.}\; Our prior covers both the DGP and the target policy, which raises a natural question: how does \methodname{} perform when the deployment policy is not represented in the policy prior? The policy prior contains history-dependent (dynamic) policies, sampled from neural networks and trees, and three types of static policies: always treat, never treat, and treat for the first few steps only. None of these produces an \emph{alternating} treatment sequence. On \texttt{mimic\_1t} with $\tau=4$, we evaluate two unseen policies, applied regardless of how the patient responds: treat at every other step, $(1,0,1,0)$, and treat at the first and last steps only, $(1,0,0,1)$. The observed data are identical to those of the original evaluation. 
Figure~\ref{fig:ablation}(c) shows that \methodname{} achieves comparable performance with the two strong baselines
on the both unseen policies. 
The result suggests that in-context conditioning on the target policy can generalize beyond the policy families represented in the prior. 

\section{Conclusion, Limitations, and Future Work}
\vspace{-5pt}
We reformulate longitudinal CAPO estimation as an ICL task by representing target policies through decisions on factual histories across the observed dataset. We introduce \methodname{}, the first foundation model to longitudinal CAPO estimation supporting both static and dynamic policies. On synthetic and semi-synthetic benchmarks, \methodname{} matches or outperforms baselines at a fraction of their deployment runtime. We further evaluate its ability to predict factual outcomes in real-world data. Further analyses demonstrate strong uncertainty calibration, greater resilience to small samples and unobserved confounders, and generalization to policies unseen during pretraining.

Several limitations motivate future work. (i) \emph{Computational complexity.} Quadratic costs in sample size $N$ and trajectory length $T$ (Appendix~F) limit scalability. Advances in scalable tabular foundation models~\citep{qu2026tabiclv,grinsztajn2026tabpfn} may inform more efficient architectures. (ii) \emph{Scope of pretraining.} Our model considers binary treatments and discrete-time trajectories without right censoring. Extending it to broader treatment spaces, continuous-time series, and censored outcomes would expand its applicability. (iii) \emph{Identifying assumptions and prior dependence.} Causal identification requires assumptions such as sequential ignorability; empirical robustness to certain violations does not establish theoretical guarantees under unobserved confounding. Accuracy may also deteriorate when deployment settings differ from the pretraining distribution. Characterizing this sensitivity and developing guarantees under prior misspecification remain important directions.


\clearpage

\bibliography{ref}
\bibliographystyle{iclr2027_conference}

\appendix
\setcounter{figure}{0}
\setcounter{table}{0}

\renewcommand{\thefigure}{S.\arabic{figure}}
\renewcommand{\thetable}{S.\arabic{table}}

\renewcommand{\theHfigure}{S.\arabic{figure}}
\renewcommand{\theHtable}{S.\arabic{table}}

\newpage
\section*{Appendix}

\section{Proof}
\PropPolicyDemoSufficiency*

\begin{proof}
\label{proof:policy_demo_sufficiency}

We first recover the deterministic policies from the population
demonstrations, then show by backward induction that they uniquely
determine the ICE recursion.

For notational convenience, write
\[
\pi_t(a\mid h_t)
:=
\mathbf{1}\{a=\pi_t(h_t)\}
\]
for the degenerate treatment distribution induced by the deterministic
rule $\pi_t$. Use the same convention for any other policy considered
below.

For each $t\in\{s,\dots,e\}$, let $P_{0}$ denote the observed
marginal distribution of $H_t$. By construction,
\[
\Pr(A_t^{\pi_t}=a\mid H_t=h_t)
=
\mathbf{1}\{a=\pi_t(h_t)\}
=
\pi_t(a\mid h_t).
\]
Thus, the joint distribution of $(H_t,A_t^{\pi_t})$ uniquely determines
$\pi_t(h_t)$, and hence $\pi_t(\cdot\mid h_t)$, for
$P_{0}$-almost every $h_t$.

To connect these policy conditionals to the target estimand,
we use the iterative conditional expectation (ICE) representation
of the longitudinal g-formula. Define
$$
Q_e^\pi(h_e,a_e)
=
\mathbb E_{P_0}
\left[
Y_{e+1}\mid H_e=h_e,A_e=a_e
\right],
$$
and
$$
V_e^\pi(h_e)
=
\sum_{a_e:\,\pi_e(a_e\mid h_e)>0}
\pi_e(a_e\mid h_e)Q_e^\pi(h_e,a_e).
$$
For $t=e-1,\dots,s$, recursively define
$$
Q_t^\pi(h_t,a_t)
=
\mathbb E_{P_0}
\left[
V_{t+1}^\pi(H_{t+1})
\mid H_t=h_t,A_t=a_t
\right],
$$
and
$$
V_t^\pi(h_t)
=
\sum_{a_t:\,\pi_t(a_t\mid h_t)>0}
\pi_t(a_t\mid h_t)Q_t^\pi(h_t,a_t).
$$
Under the stated longitudinal identification assumptions ~\citep{shirakawa2024longitudinal,wen2021parametric,hess2026igc},
the g-formula gives
$$
V_s^\pi(h_s)
=
\mathbb E
\left[
Y_{e+1}^{\pi_{s:e}}\mid H_s=h_s
\right]
=
\psi^{\pi_{s:e}}(h_s)
$$
for almost every $h_s$ under $P_{0}$.

It therefore suffices to establish uniqueness of $V_s^\pi$.
Consider two policy sequences $\pi$ and $\widetilde\pi$ that
generate the same population demonstration distributions.
Let $(Q_t^\pi,V_t^\pi)$ and
$(Q_t^{\tilde \pi},V_t^{\tilde \pi})$ denote their corresponding
ICE recursions under the same $P_0$. Policy recovery implies
$$
\pi_t(\cdot\mid h_t)
=
\widetilde\pi_t(\cdot\mid h_t)
\qquad P_{0}\text{-a.e.}
$$
We show by backward induction that
$V_t^\pi=V_t^{\tilde \pi}$ almost everywhere under $P_{0}$.

For the base case $t=e$, both terminal regressions are conditional
expectations of the same observed outcome under $P_0$. Hence,
$$
Q_e^\pi(H_e,A_e)
=
Q_e^{\tilde \pi}(H_e,A_e)
\qquad P_0\text{-a.s.}
$$
By positivity,
$$
\pi_e(a_e\mid h_e)>0
\implies
\pi_e^*(a_e\mid h_e)>0.
$$
Consequently, for almost every $h_e$ under $P_{0}$,
the terminal regressions agree on every action contributing
to the policy-weighted sum. Together with
$\pi_e=\widetilde\pi_e$ almost everywhere, this gives
$$
V_e^\pi(h_e)=V_e^{\tilde \pi}(h_e)
\qquad P_{0}\text{-a.e.}
$$

For the induction step, suppose
$V_{t+1}^\pi=V^{\tilde \pi}_{t+1}$ almost everywhere under
$P_{0}$. We first show that any remaining disagreement
on an observed-history null set cannot affect $Q_t^\pi$.
Define
$$
N_{t+1}
=
\{h_{t+1}:V_{t+1}^\pi(h_{t+1})
\neq V^{\tilde \pi}_{t+1}(h_{t+1})\}.
$$
The induction hypothesis implies
$$
\begin{aligned}
P_0(H_{t+1}\in N_{t+1})=0
&\Longrightarrow
P_0(H_{t+1}\in N_{t+1}\mid H_t,A_t)=0
\quad P_0\text{-a.s.}
\\
&\Longrightarrow
Q_t^\pi(H_t,A_t)
=
Q_t^{\tilde \pi}(H_t,A_t)
\quad P_0\text{-a.s.}
\end{aligned}
$$
The second implication follows because $Q_t^\pi$ and
$Q_t^{\tilde \pi}$ average $V_{t+1}^\pi$ and
$V_{t+1}^{\tilde \pi}$ under the same observed conditional
distribution.

We then apply the same positivity argument as in the base case:
the regressions agree on all actions contributing to the
policy-weighted sum, for almost every observed history.
Since $\pi_t=\widetilde\pi_t$ almost everywhere,
$$
V_t^\pi(h_t)
=
V_t^{\tilde \pi}(h_t)
\qquad P_{0}\text{-a.e.}
$$
This completes the induction.

In particular,
$$
V_s^\pi(h_s)= V_s^{\tilde \pi}(h_s)
\qquad P_{0}\text{-a.e.}
$$
Since the ICE identification result equates $V_s^\pi$ with
the target functional, $P_0$ and the population policy
demonstration distributions uniquely determine
$\Psi^{\pi_{s:e}}(P_0)(h_s)$ for almost every $h_s$
under $P_{0}$.
\end{proof}

\section{Implementation Details}
\subsection{Model Architecture}\label{appx:architecture}
The model operates on two kinds of tokens: \emph{observed data tokens}, which represent the factual trajectories, and \emph{readout tokens}, which combine the factual representations with target policy demonstrations to predict potential outcomes. Both are processed by the same stack of eight transformer layers. Each layer has two parts: an \emph{observed data backbone} that updates the observed data tokens, followed by a \emph{readout update} in which the readout tokens attend to the backbone's output at that layer. 
The backbone never attends to readout tokens, so its representations are independent of the target policy and intervention window. 
We first describe the architecture for a single intervention window $(s,e)$, and then explain how the same construction is extended to all admissible windows in parallel.

\paragraph{Feature Pooling and Tokenization.}
Each trajectory is represented by $2T$ pooled embeddings corresponding to the interleaved sequence $L_1,A_1,\dots,L_T,A_T$. To encode each $L$ or $A$ vector, we first embed its scalar features and add learned column-identity embeddings. A feature-attention layer models interactions among features using self-attention, followed by attention pooling with learned query tokens to aggregate them into a single embedding~\citep{lee2019set, qu2025tabicl, grinsztajn2026tabpfn}. Baseline covariates $B$ are encoded in the same way, and their pooled embedding is added to every $L_t$. Each sequence position also receives a position embedding over the $2T$ positions plus an $L$/$A$ type embedding.

The readout tokens $R_{s:e}$ are initialized with three components: (i) Start- and end-index embeddings; (ii) The initial target policy demonstration $A^{\pi_s}_{s}$; (iii) Outcome: factual outcome $Y_{e+1}$ for context trajectories, and a null value for query trajectories.

\paragraph{Observed data backbone.} 
Within each transformer layer, the backbone updates the observed data tokens by (i) causal attention\footnote{Here "causal attention" refers only to temporally masked attention: each token may attend to the past and present, but not the future. Following common usage, we retain this term, although the layer does not itself encode or identify causal relationships among variables.} along time, allowing $L_t$ to attend to $(B, L_{1:t},A_{1:t-1})$ and $A_t$ to attend to $(B, L_{1:t},A_{1:t})$; (ii) row attention across trajectories at each time position, with keys and values restricted to context trajectories. The output then goes through an MLP layer.

\paragraph{Window-specific readout.} 
Within each transformer layer, after the backbone update, the readout token is updated in three steps. (i) $R_{s:e}$ cross-attends to the backbone representations of its own trajectory at that layer, under an asymmetric, window-specific mask. A query readout accesses only the tokens representing $H_{s}$, whereas a context readout accesses the factual sequence through the end of the window, $(H_e,A_e)$; all later backbone tokens are masked. Context readouts additionally cross-attend to embeddings of the target policy demonstrations $\{A^{\pi_t}_{t}\}_{t=s}^{e}$. We empirically found that additionally supplying the differences between demonstrated and factual treatments, $(A^{\pi_t}_{t}-A_{t})_{t=s}^{e}$ accelerates convergence without changing the information available to the model(Figure~\ref{fig:residual_ablation}). (ii) The readout tokens attend across trajectories, with keys and values restricted to context readouts. (iii) An MLP follows. The backbone and the readout update have separate parameters. After the last layer, a linear head maps each query readout to a bar distribution~\citep{hollmann2025accurate} over the standardized outcome.

\begin{figure}[t]
    \centering
    \includegraphics[width=0.8\linewidth]{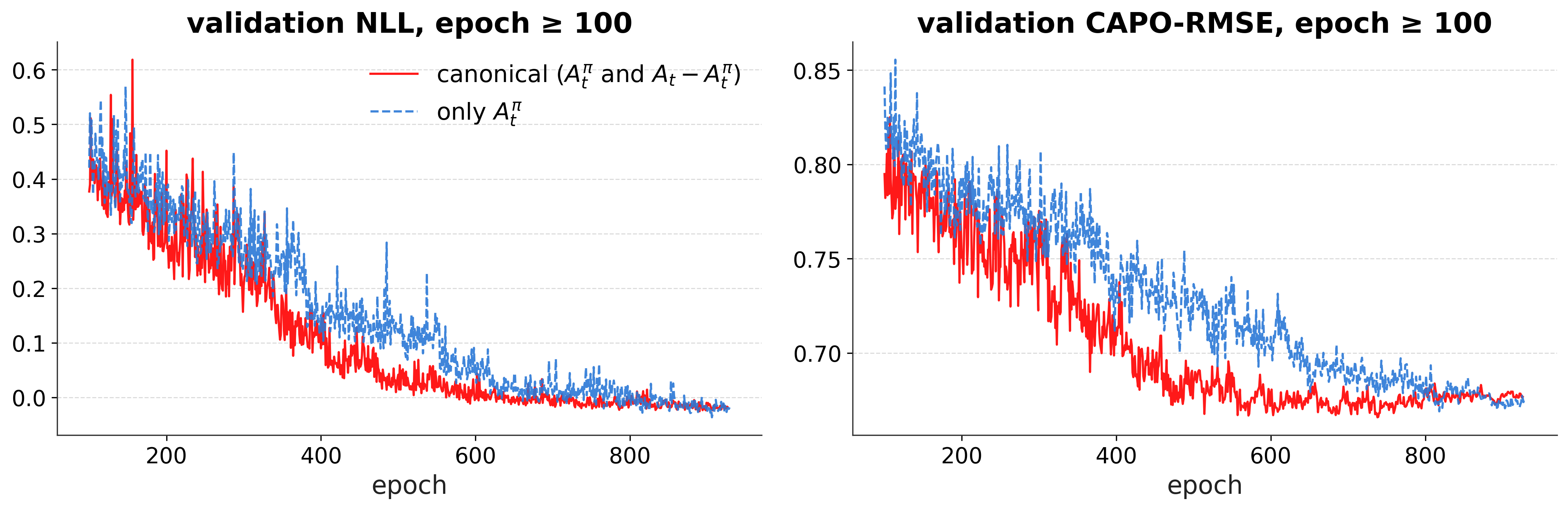}
    \caption{Ablation study on the target policy demonstration channel. Supplying the differences between demonstrated and factual treatments accelerates convergence.}
    \label{fig:residual_ablation}
\end{figure}

\paragraph{Dense CAPO prediction across windows.}
The description above considers a single window $(s,e)$. We instantiate one readout token per trajectory for every admissible window. These tokens share the same parameters and differ only in their initialization and window-specific access masks. Readouts of different windows do not interact: cross-trajectory attention is performed separately for each window, so a readout for $(s,e)$ attends only to context readouts for that same $(s,e)$.

\subsection{Prior Sampling}
\label{appx:prior}
We pretrain exclusively on synthetic tasks, each defined by a longitudinal SCM and a target policy. Each task contains 1024 trajectories, of which the first 80\% serve as context and the remaining 20\% as query trajectories. The SCM specifies the factual data-generating process, including the behavioural policy through its treatment-assignment mechanisms. The target policy may be a prescribed treatment sequence or a dynamic policy that depends on the observed history. We sample each task in five stages.

\begin{enumerate}
\item \textbf{Dimension and variable types} We sample the numbers of baseline covariates, time-varying covariates, and binary treatment dimensions as $d_B\in\{1,\dots,32\}$, $d_L\in\{1,\dots,32\}$, and $d_A\in\{1,2,3\}$, respectively. Time-varying covariates may be continuous or ordinal, while baseline covariates may additionally be nominal. 


\item \textbf{Skeleton.} 
The SCM starts with the baseline variables $(B,L_1)$, which are generated either jointly through a random DAG or with $B$ independent of $L_1$.
For each coordinate of the subsequent variables $(A_{1:T},L_{2:T},Y_{2:T+1})$, we first sample a time-invariant set of \emph{parent variables}. For each selected parent variable, we then sample a set of lag values, with the number of lags drawn as $1+\operatorname{Poisson}(\lambda)$ and lag values selected using geometrically decaying weights that favour shorter lags. Each resulting variable--lag pair specifies a \emph{parent node} in the longitudinal SCM.
Each $L$ coordinate and $Y$ additionally includes its own preceding-step value as a parent node.
We constrain the resulting longitudinal graph to be acyclic: cross-step edges point forward in time, contemporaneous edges among $L$ follow an acyclic ordering, and treatment nodes depend only on preceding nodes in history $H_t$. 

\item \textbf{SCM mechanisms and behavioural policy.} We draw random functions from seven families: neural networks, decision trees, discretizations~\citep{hollmann2025accurate}, linear functions, quadratic functions with pairwise interactions, smooth functions sampled from a Gaussian process, and random activation functions (sigmoid, softplus, or exponential)~\citep{qu2026tabiclv}. For covariate and outcome coordinates, we explicitly incorporate persistence through a per-coordinate weight sampled from $U(0.5,0.95)$ on the preceding-step value, together with a smooth time trend. For a coordinate with treatment parents, its mechanism takes the form $g_0(x)+\sum_i a_i g_i(x),$ where $x$ denotes its non-treatment parents and each $a_i\in\{0,1\}$ denotes a treatment parent at a particular dimension and lag. The functions $g_i$ are drawn independently, allowing treatment effects to vary with the covariates. The treatment-assignment mechanisms define the behavioural policy. For treatment dimension $d$, we set \[ P(A_{d,t}=1\mid H_t) = \sigma\!\left( \frac{z_d(H_t)+w_d(2A_{d,t-1}-1)-c_d}{s_d} \right), \] where $z_d$ is a randomly drawn function of the history, $w_d\geq 0$ controls treatment persistence, and $c_d$ controls treatment prevalence. We draw $s_d\sim\operatorname{LogUniform}(0.05,5)$ to span assignment mechanisms ranging from nearly deterministic to nearly randomised, and vary $c_d$ to produce treatment prevalences ranging from rare to common. Finally, we clip the assignment probabilities to $[0.01,0.99]$ to ensure positivity.

\item \textbf{Target policy.} For each treatment dimension, we sample with equal probability either a history-dependent policy with its own random mechanism or a static treatment sequence. The history-dependent policy is constructed with its own randomly sampled parents and mechanism drawn the same family as the behavioural policy. Static sequences include always treating, never treating, and treating only for the first $k\in\{1,\dots,\tau-1\}$ steps. In half of the tasks, we subsequently blend the sampled target-policy probabilities toward the behavioural probabilities, using a mixing weight drawn from $U(0.3,1)$. After any blending with the behavioral probabilities, we threshold the resulting probabilities 
to obtain a deterministic target-policy rule. This strategy produce target policies that remain close to the behavioural policy.

\item \textbf{Rollout and supervision.} We first generate factual trajectories from the sampled SCM, drawing exogenous noise for the baseline variables and subsequent transitions. For each intervention start step $s$, we retain the observed prefix and roll forward under the target policy for up to $\tau_{\max}$ steps, reusing the same exogenous noise realisations. The resulting single-draw counterfactual outcomes serve as training labels. For a fixed SCM and target policy, their conditional expectation given $H_s$ is the CAPO.

\end{enumerate}

\begin{table}[t]
\centering
\caption{\methodname{} Architecture Hyperparameters.}
\label{tab:arch}
\begin{tabular}{ll}
\toprule
\textbf{Component} & \textbf{Value} \\
\midrule
Transformer layers & 8 \\
Embedding width & 128 \\
Attention heads & 4 \\
MLP hidden width & 512 \\
Activation & GELU \\
Dropout & 0 \\
Learned pooling queries for feature pooling & 4 \\
Maximum sequence length $T$ & 15 \\
Maximum horizon $\tau_{\max}$ & 10 \\
Output buckets (bar distribution) & 100 \\
Total parameters & 5.4 M \\
\bottomrule
\end{tabular}
\end{table} 

\subsection{Hyperparameters and Training}\label{appx:hparams}

Table~\ref{tab:arch} lists the hyperparameters of the model described in Appendix~\ref{appx:architecture}.
Every attention module, including the feature-attention layer and the attention pooling, uses the same number of heads. The observed data backbone and the readout update have separate parameters in each of the eight layers; a final layer normalization is applied to the readout tokens before the output head. The model has 5.4M parameters, 4.8M of which lie in the transformer layers and the remainder in the tokenizers, embeddings, and output head. We train with AdamW. The learning rate is warmed up linearly to 0.001 and then halved whenever the validation loss has not improved for 40 epochs. Model selection uses a fixed set of 128 held-out validation tasks drawn from the prior; we report the checkpoint with the lowest validation loss. At inference, a task is processed in a single forward pass with no per-task fitting.

\section{Additional Results}
Full CAPO RMSE results across all methods and all horizons are provided in Table~\ref{tab:full_benchmark_per_horizons}. Detailed runtime results across horizons are in Table~\ref{tab:runtime}. 
\begin{table*}[t]
\centering
\caption{\textbf{Longitudinal CAPO benchmark, per horizon.} Each cell reports the normalized RMSE, mean$\pm$SD across the 20 repeats; values in parentheses are the mean$\pm$SD of the method's per-task RMSE ranking ($\downarrow$ better). The \textbf{best} and \underline{second best} methods per benchmark and horizon are highlighted.}
\label{tab:full_benchmark_per_horizons}
\resizebox{\textwidth}{!}{%
\begin{tabular}{@{}llccccc@{}}
\toprule
Benchmark & Method & $\tau=2$ & $\tau=4$ & $\tau=6$ & $\tau=8$ & $\tau=10$ \\
\midrule
\multirow{7}{*}{Cancer} & \textbf{PICLE-PFN} & \underline{0.31$\pm$0.03 (2.3$\pm$0.6)} & \textbf{0.26$\pm$0.04 (1.6$\pm$0.8)} & \textbf{0.18$\pm$0.03 (1.1$\pm$0.3)} & \textbf{0.12$\pm$0.03 (1.1$\pm$0.3)} & \textbf{0.08$\pm$0.02 (1.1$\pm$0.3)} \\
 & IGC-Net & 0.32$\pm$0.02 (2.6$\pm$0.5) & 0.29$\pm$0.04 (2.5$\pm$0.8) & 0.27$\pm$0.03 (3.3$\pm$0.7) & 0.44$\pm$0.08 (4.8$\pm$0.5) & 0.16$\pm$0.07 (3.4$\pm$1.2) \\
 & CT & \textbf{0.24$\pm$0.02 (1.1$\pm$0.2)} & \underline{0.27$\pm$0.03 (1.9$\pm$0.7)} & \underline{0.24$\pm$0.06 (2.4$\pm$1.0)} & \underline{0.19$\pm$0.09 (2.5$\pm$1.1)} & \underline{0.15$\pm$0.12 (2.8$\pm$1.3)} \\
 & RMSN & 0.55$\pm$0.03 (4.7$\pm$0.5) & 0.50$\pm$0.07 (5.0$\pm$0.2) & 0.35$\pm$0.06 (4.5$\pm$0.8) & 0.23$\pm$0.05 (3.2$\pm$0.7) & 0.16$\pm$0.04 (3.7$\pm$0.9) \\
 & G-Net & 0.75$\pm$0.03 (6.0$\pm$0.2) & 0.68$\pm$0.04 (6.0$\pm$0.0) & 0.63$\pm$0.05 (6.0$\pm$0.0) & 0.62$\pm$0.05 (6.0$\pm$0.0) & 0.62$\pm$0.06 (6.0$\pm$0.0) \\
 & CRN & 0.50$\pm$0.12 (4.4$\pm$0.8) & 0.40$\pm$0.06 (4.0$\pm$0.2) & 0.31$\pm$0.05 (3.7$\pm$1.0) & 0.23$\pm$0.05 (3.4$\pm$0.9) & 0.18$\pm$0.05 (4.0$\pm$0.8) \\
 & TE-CDE & 0.84$\pm$0.02 (7.0$\pm$0.2) & 1.02$\pm$0.07 (7.0$\pm$0.0) & 1.17$\pm$0.10 (7.0$\pm$0.0) & 1.27$\pm$0.12 (7.0$\pm$0.0) & 1.31$\pm$0.12 (7.0$\pm$0.0) \\
\midrule
\multirow{7}{*}{MIMIC\_Y1} & \textbf{PICLE-PFN} & 0.58$\pm$0.33 (3.2$\pm$1.4) & 0.82$\pm$0.43 (3.0$\pm$1.5) & 0.97$\pm$0.45 (3.2$\pm$1.5) & 1.14$\pm$0.62 (3.1$\pm$1.4) & \underline{1.28$\pm$0.88 (2.5$\pm$1.5)} \\
 & IGC-Net & \underline{0.50$\pm$0.15 (2.2$\pm$1.4)} & \underline{0.70$\pm$0.27 (2.4$\pm$1.5)} & \underline{0.86$\pm$0.45 (2.1$\pm$1.3)} & \textbf{1.02$\pm$0.68 (2.1$\pm$1.3)} & 1.43$\pm$0.99 (3.1$\pm$1.5) \\
 & CT & \textbf{0.47$\pm$0.15 (2.0$\pm$1.0)} & \textbf{0.69$\pm$0.23 (2.5$\pm$1.1)} & \textbf{0.85$\pm$0.28 (2.5$\pm$1.1)} & \underline{1.03$\pm$0.38 (2.9$\pm$1.3)} & \textbf{1.22$\pm$0.46 (3.1$\pm$1.6)} \\
 & RMSN & 0.66$\pm$0.26 (4.8$\pm$1.3) & 0.99$\pm$0.49 (4.7$\pm$1.3) & 1.31$\pm$0.76 (4.7$\pm$1.7) & 1.56$\pm$1.01 (4.8$\pm$1.8) & 1.69$\pm$1.12 (4.5$\pm$2.1) \\
 & G-Net & 0.91$\pm$0.33 (6.3$\pm$1.2) & 1.17$\pm$0.64 (6.0$\pm$1.4) & 1.38$\pm$0.92 (5.2$\pm$1.6) & 1.55$\pm$1.16 (4.8$\pm$1.7) & 1.66$\pm$1.28 (4.5$\pm$1.7) \\
 & CRN & 0.54$\pm$0.16 (3.5$\pm$1.2) & 0.83$\pm$0.32 (3.8$\pm$1.7) & 1.09$\pm$0.40 (4.3$\pm$2.0) & 1.32$\pm$0.45 (4.7$\pm$2.2) & 1.48$\pm$0.45 (4.9$\pm$2.3) \\
 & TE-CDE & 0.86$\pm$0.41 (6.0$\pm$1.0) & 1.09$\pm$0.31 (5.6$\pm$1.7) & 1.37$\pm$0.37 (5.7$\pm$1.5) & 1.57$\pm$0.49 (5.5$\pm$1.6) & 1.72$\pm$0.63 (5.1$\pm$1.6) \\
\midrule
\multirow{7}{*}{MIMIC\_Y2} & \textbf{PICLE-PFN} & \underline{0.36$\pm$0.14 (1.9$\pm$0.9)} & \textbf{0.48$\pm$0.18 (1.8$\pm$1.0)} & \textbf{0.61$\pm$0.23 (1.9$\pm$1.1)} & \textbf{0.73$\pm$0.31 (1.8$\pm$1.4)} & \textbf{0.83$\pm$0.40 (1.7$\pm$1.3)} \\
 & IGC-Net & \textbf{0.34$\pm$0.15 (1.8$\pm$0.8)} & \underline{0.49$\pm$0.18 (2.0$\pm$0.9)} & \underline{0.62$\pm$0.20 (2.0$\pm$1.5)} & \underline{0.75$\pm$0.20 (2.4$\pm$1.0)} & \underline{0.83$\pm$0.26 (2.3$\pm$1.4)} \\
 & CT & 0.51$\pm$0.46 (3.0$\pm$1.5) & 0.73$\pm$0.71 (3.2$\pm$1.6) & 0.92$\pm$0.91 (3.2$\pm$1.6) & 1.08$\pm$0.99 (3.3$\pm$1.7) & 1.19$\pm$1.03 (3.6$\pm$1.6) \\
 & RMSN & 0.60$\pm$0.12 (5.0$\pm$1.3) & 0.68$\pm$0.14 (4.5$\pm$1.2) & 0.84$\pm$0.18 (4.5$\pm$1.3) & 0.97$\pm$0.19 (4.5$\pm$1.1) & 1.06$\pm$0.21 (4.3$\pm$1.3) \\
 & G-Net & 0.79$\pm$0.09 (6.6$\pm$0.7) & 0.84$\pm$0.15 (5.8$\pm$1.1) & 0.96$\pm$0.16 (5.5$\pm$1.2) & 1.10$\pm$0.16 (5.5$\pm$1.2) & 1.24$\pm$0.18 (5.4$\pm$1.1) \\
 & CRN & 0.55$\pm$0.28 (4.3$\pm$1.5) & 0.70$\pm$0.29 (4.3$\pm$1.4) & 0.85$\pm$0.37 (4.5$\pm$1.1) & 0.99$\pm$0.44 (4.0$\pm$1.4) & 1.13$\pm$0.51 (4.1$\pm$1.4) \\
 & TE-CDE & 0.65$\pm$0.24 (5.3$\pm$0.9) & 0.93$\pm$0.23 (6.4$\pm$0.8) & 1.26$\pm$0.39 (6.4$\pm$1.0) & 1.61$\pm$0.57 (6.6$\pm$0.8) & 1.96$\pm$0.80 (6.5$\pm$0.9) \\
\midrule
\multirow{10}{*}{MIMIC\_1T} & \textbf{PICLE-PFN} & 0.38$\pm$0.13 (3.1$\pm$2.1) & \underline{0.54$\pm$0.22 (2.6$\pm$2.2)} & \textbf{0.62$\pm$0.20 (2.1$\pm$1.8)} & \textbf{0.67$\pm$0.26 (1.9$\pm$1.9)} & \textbf{0.70$\pm$0.24 (1.9$\pm$1.9)} \\
 & IGC-Net & 0.44$\pm$0.31 (3.7$\pm$2.5) & \textbf{0.52$\pm$0.13 (3.7$\pm$1.6)} & \underline{0.62$\pm$0.14 (3.5$\pm$2.0)} & 1.03$\pm$0.40 (6.9$\pm$2.5) & 0.88$\pm$0.15 (5.2$\pm$2.5) \\
 & CT & 0.37$\pm$0.20 (2.2$\pm$2.0) & 0.57$\pm$0.37 (2.5$\pm$2.0) & 0.71$\pm$0.47 (3.2$\pm$2.2) & 0.79$\pm$0.48 (3.0$\pm$1.9) & 0.85$\pm$0.48 (3.6$\pm$1.8) \\
 & RMSN & 0.47$\pm$0.13 (5.6$\pm$2.0) & 0.70$\pm$0.23 (6.2$\pm$2.2) & 0.85$\pm$0.27 (6.8$\pm$1.8) & 0.92$\pm$0.28 (6.2$\pm$1.8) & 0.99$\pm$0.29 (6.0$\pm$1.5) \\
 & G-Net & 0.93$\pm$0.08 (9.0$\pm$0.9) & 1.08$\pm$0.21 (8.9$\pm$1.2) & 1.21$\pm$0.29 (8.6$\pm$1.5) & 1.21$\pm$0.25 (8.4$\pm$1.6) & 1.21$\pm$0.21 (8.4$\pm$1.6) \\
 & CRN & \textbf{0.36$\pm$0.07 (3.5$\pm$1.4)} & 0.60$\pm$0.22 (5.0$\pm$2.1) & 0.87$\pm$0.41 (6.8$\pm$1.7) & 1.07$\pm$0.54 (6.9$\pm$1.9) & 1.20$\pm$0.61 (7.0$\pm$2.0) \\
 & TE-CDE & 0.96$\pm$0.10 (9.7$\pm$0.6) & 1.08$\pm$0.14 (9.2$\pm$1.0) & 1.21$\pm$0.23 (8.8$\pm$1.4) & 1.23$\pm$0.28 (8.2$\pm$1.7) & 1.24$\pm$0.34 (8.0$\pm$2.0) \\
 & D3 & 0.41$\pm$0.09 (5.8$\pm$1.2) & 0.55$\pm$0.10 (4.7$\pm$1.2) & 0.64$\pm$0.12 (3.6$\pm$1.2) & \underline{0.70$\pm$0.14 (3.2$\pm$1.4)} & \underline{0.81$\pm$0.32 (3.0$\pm$2.2)} \\
 & DLTMLE & 0.55$\pm$0.17 (7.5$\pm$1.0) & 0.67$\pm$0.16 (7.0$\pm$1.5) & 0.94$\pm$0.49 (6.8$\pm$2.0) & 1.04$\pm$0.36 (6.8$\pm$1.8) & 1.06$\pm$0.34 (6.7$\pm$2.0) \\
 & DeepACE & \underline{0.37$\pm$0.08 (4.8$\pm$1.4)} & 0.66$\pm$0.41 (5.1$\pm$2.9) & 0.72$\pm$0.35 (4.7$\pm$2.6) & 0.72$\pm$0.18 (3.4$\pm$1.8) & 0.96$\pm$0.49 (4.9$\pm$2.6) \\
\midrule
\multirow{4}{*}{MIMIC\_1T\_HDP0.5} & \textbf{PICLE-PFN} & \textbf{0.32$\pm$0.10 (1.1$\pm$0.3)} & \textbf{0.46$\pm$0.14 (1.1$\pm$0.4)} & \textbf{0.56$\pm$0.15 (1.6$\pm$0.5)} & \textbf{0.62$\pm$0.15 (1.4$\pm$0.7)} & \textbf{0.68$\pm$0.14 (1.6$\pm$0.8)} \\
 & IGC-Net & \underline{0.33$\pm$0.09 (2.0$\pm$0.6)} & \underline{0.48$\pm$0.13 (2.0$\pm$0.7)} & \underline{0.57$\pm$0.14 (1.6$\pm$0.8)} & \underline{0.66$\pm$0.16 (2.2$\pm$0.6)} & 0.70$\pm$0.14 (2.4$\pm$0.8) \\
 & D3 & 0.35$\pm$0.10 (2.9$\pm$0.3) & 0.51$\pm$0.15 (2.9$\pm$0.4) & 0.62$\pm$0.17 (2.8$\pm$0.4) & 0.66$\pm$0.17 (2.5$\pm$0.8) & \underline{0.69$\pm$0.14 (2.0$\pm$0.8)} \\
 & DLTMLE & 0.39$\pm$0.11 (4.0$\pm$0.2) & 0.56$\pm$0.16 (3.9$\pm$0.3) & 0.69$\pm$0.19 (4.0$\pm$0.0) & 0.81$\pm$0.21 (4.0$\pm$0.0) & 0.86$\pm$0.19 (4.0$\pm$0.0) \\
\midrule
\multirow{4}{*}{MIMIC\_1T\_HDP1} & \textbf{PICLE-PFN} & \textbf{0.32$\pm$0.10 (1.1$\pm$0.4)} & \textbf{0.46$\pm$0.14 (1.1$\pm$0.2)} & \textbf{0.56$\pm$0.15 (1.2$\pm$0.4)} & \textbf{0.63$\pm$0.15 (1.3$\pm$0.7)} & \textbf{0.69$\pm$0.14 (1.4$\pm$0.6)} \\
 & IGC-Net & \underline{0.34$\pm$0.10 (2.1$\pm$0.7)} & 0.56$\pm$0.13 (2.9$\pm$1.0) & \underline{0.59$\pm$0.15 (1.9$\pm$0.7)} & \underline{0.68$\pm$0.16 (2.4$\pm$0.6)} & 0.73$\pm$0.16 (2.7$\pm$0.5) \\
 & D3 & 0.35$\pm$0.10 (2.8$\pm$0.4) & \underline{0.53$\pm$0.15 (2.5$\pm$0.5)} & 0.65$\pm$0.18 (3.0$\pm$0.5) & 0.69$\pm$0.17 (2.4$\pm$0.7) & \underline{0.72$\pm$0.16 (1.9$\pm$0.8)} \\
 & DLTMLE & 0.40$\pm$0.12 (4.0$\pm$0.2) & 0.58$\pm$0.17 (3.5$\pm$0.6) & 0.73$\pm$0.20 (3.9$\pm$0.4) & 0.84$\pm$0.21 (4.0$\pm$0.0) & 0.88$\pm$0.21 (4.0$\pm$0.0) \\
\midrule
\multirow{4}{*}{MIMIC\_1T\_HDP2} & \textbf{PICLE-PFN} & \textbf{0.33$\pm$0.10 (1.1$\pm$0.4)} & \textbf{0.48$\pm$0.14 (1.0$\pm$0.0)} & \textbf{0.58$\pm$0.16 (1.4$\pm$0.6)} & \textbf{0.66$\pm$0.17 (1.4$\pm$0.7)} & \textbf{0.73$\pm$0.17 (1.4$\pm$0.6)} \\
 & IGC-Net & \underline{0.34$\pm$0.10 (2.0$\pm$0.7)} & 0.58$\pm$0.15 (2.9$\pm$0.9) & \underline{0.64$\pm$0.20 (2.0$\pm$0.9)} & 0.78$\pm$0.25 (2.6$\pm$0.7) & \underline{0.77$\pm$0.21 (2.0$\pm$0.9)} \\
 & D3 & 0.37$\pm$0.10 (2.9$\pm$0.4) & \underline{0.55$\pm$0.16 (2.5$\pm$0.6)} & 0.70$\pm$0.20 (2.8$\pm$0.6) & \underline{0.73$\pm$0.19 (2.0$\pm$0.8)} & 0.82$\pm$0.23 (2.7$\pm$0.8) \\
 & DLTMLE & 0.41$\pm$0.11 (4.0$\pm$0.2) & 0.61$\pm$0.17 (3.6$\pm$0.5) & 0.78$\pm$0.22 (3.9$\pm$0.3) & 0.92$\pm$0.27 (4.0$\pm$0.2) & 0.95$\pm$0.28 (3.8$\pm$0.5) \\
\bottomrule
\end{tabular}}
\end{table*}

\begin{table*}[t]
\centering
\caption{\textbf{Runtime(seconds) at deployment.} One evaluation task per setting (1000 context, 1000 queries), one A100 GPU, mean$\pm$SD over 5 timing repeats. Baselines are measured on their tuned HPs. CT, RMSN, G-Net, CRN, and TE-CDE is trained once then evaluated on all five horizons; runtime is computed as total time of one training and five inference calls, divided by five); IGC-net, D3, DLTMLE, and DeepACE train one model per horizon; runtime is measured as one training and one inference call. PICLE-PFN is pretrained and only runs inference.} 
\label{tab:runtime}
\resizebox{\textwidth}{!}{%
\begin{tabular}{@{}llrrrrrrrrr@{}}
\toprule
\multirow{2}{*}{Method} & \multirow{2}{*}{Type} & \multicolumn{3}{c}{Cancer} & \multicolumn{3}{c}{MIMIC\_Y1} & \multicolumn{3}{c}{MIMIC\_1T} \\
\cmidrule(lr){3-5} \cmidrule(lr){6-8} \cmidrule(lr){9-11}
 & & train & infer & total & train & infer & total & train & infer & total \\
\midrule
\textbf{PICLE-PFN} & amortised & --- & 0.13$\pm$0.00 & 0.13$\pm$0.00 & --- & 0.14$\pm$0.00 & 0.14$\pm$0.00 & --- & 0.14$\pm$0.00 & 0.14$\pm$0.00 \\
IGC-Net & fit per horizon & 525$\pm$11 & 0.08$\pm$0.00 & 525$\pm$11 & 582$\pm$5 & 0.08$\pm$0.01 & 582$\pm$5 & 725$\pm$10 & 0.09$\pm$0.02 & 725$\pm$10 \\
CT & fit once & 23.84$\pm$0.57 & 0.25$\pm$0.03 & 24.09$\pm$0.57 & 78.82$\pm$0.47 & 1.08$\pm$0.02 & 79.90$\pm$0.46 & 66.46$\pm$2.04 & 1.26$\pm$0.11 & 67.72$\pm$1.96 \\
RMSN & fit once & 65.62$\pm$1.36 & 0.10$\pm$0.01 & 65.71$\pm$1.35 & 33.96$\pm$0.55 & 0.07$\pm$0.00 & 34.03$\pm$0.55 & 6.26$\pm$0.49 & 0.08$\pm$0.00 & 6.34$\pm$0.50 \\
G-Net & fit once & 4.31$\pm$0.42 & 6.58$\pm$0.73 & 10.88$\pm$0.47 & 12.00$\pm$0.45 & 6.62$\pm$0.18 & 18.62$\pm$0.40 & 4.87$\pm$0.37 & 8.45$\pm$0.56 & 13.32$\pm$0.41 \\
CRN & fit once & 32.85$\pm$0.34 & 0.10$\pm$0.01 & 32.95$\pm$0.34 & 4.97$\pm$0.16 & 0.13$\pm$0.00 & 5.10$\pm$0.17 & 21.64$\pm$0.63 & 0.13$\pm$0.03 & 21.77$\pm$0.62 \\
TE-CDE & fit once & 7.76$\pm$0.53 & 0.13$\pm$0.01 & 7.90$\pm$0.53 & 52.86$\pm$0.29 & 0.11$\pm$0.00 & 52.97$\pm$0.29 & 15.46$\pm$0.46 & 0.12$\pm$0.00 & 15.58$\pm$0.46 \\
D3 & fit per horizon & --- & --- & --- & --- & --- & --- & 12.29$\pm$0.52 & 0.14$\pm$0.02 & 12.42$\pm$0.53 \\
DLTMLE & fit per horizon & --- & --- & --- & --- & --- & --- & 7.85$\pm$0.30 & 0.14$\pm$0.02 & 7.98$\pm$0.28 \\
DeepACE & fit per horizon & --- & --- & --- & --- & --- & --- & 12.48$\pm$0.45 & 0.02$\pm$0.00 & 12.51$\pm$0.45 \\
\bottomrule
\end{tabular}}
\end{table*}

\section{Experiment Details}

\subsection{Benchmark Datasets}\label{appx:benchmark_data}
Our benchmark settings are adapted from \cite{hess2026igc}, including one fully synthetic setting and six semi-synthetic setting based on MIMIC-III. For each benchmark, we generate 20 datasets, each containing
1000 training (context) trajectories and 1000 test (query)
trajectories of length $T = 15$. We estimate CAPO on the query trajectories at five horizons, $\tau\in\{2,4,6,8,10\}$, with the intervention start time fixed at $s=2$. RMSE is first divided by the factual outcome in each synthesized dataset, then macro-averaged across 20 datasets. Mean$\pm$SD of the normalized RMSE is reported.

\subsubsection{\texttt{cancer}}
A fully synthetic tumor-growth simulator commonly used in causal inference literature~\citep{hess2026igc,bica2020estimating,li2021g,lim2018forecasting,melnychuk2022causal}. It has a single time-varying covariate (tumor volume), two binary treatment dimensions (chemotherapy, radiotherapy), and one baseline covariate (patient type). The target policy in this setting is a \emph{static} always-combined-treatment. 
The outcome, tumor volume $Y_t$, evolves as
\begin{equation*}
    Y_{t+1}
    =
    Y_t \left[
        1 + \rho \log\left(\frac{K}{Y_t}\right)
        - \alpha_c c_t
        - \left(\alpha_r d_t + \beta_r d_t^2\right)
        + \epsilon_t
    \right],
\end{equation*}
where $\rho$ and $K$ govern tumor growth and carrying capacity,
respectively. Treatment effects depend on the chemotherapy
drug concentration $c_t$ and radiation dose $d_t$, with
sensitivity parameters $\alpha_c$, $\alpha_r$, and $\beta_r$.
The term $\epsilon_t$ introduces stochastic variation.

Treatment assignment depends on recent tumor progression,
introducing time-varying confounding. Specifically, the binary
chemotherapy and radiotherapy indicators are sampled as
\begin{equation*}
    A_t^c, A_t^r
    \sim
    \operatorname{Bernoulli}\left[
        \sigma\left(
            \frac{\gamma}{D_{\max}}
            \left\{
                \bar D_{15}(\bar Y_{t-1}
                - \frac{D_{\max}}{2})
            \right\}
        \right)
    \right],
\end{equation*}
where $\sigma$ is the logistic function, $D_{\max}$ is the maximum tumour volume, and
$\bar D_{15}(\bar Y_{t-1})$ denotes the mean tumor diameter
over the preceding 15 time steps.
$\gamma$ controls how strongly treatment assignment
depends on tumor history. The target policy is always treating with both chemotherapy and radiation.
We adopt the simulator parameterization of \citet{hess2026igc}, to which we refer
for further implementation details.

\subsubsection{\texttt{mimic\_y1}, \texttt{mimic\_y2}}
These benchmarks, adapted from \cite{hess2026igc}, combine real-world clinical covariate trajectories with simulated
treatments and outcomes. The real-world covariates are obtained using
MIMIC-Extract~\citep{Wang_2020} and comprise 25 time-varying
vital signs, together with gender, ethnicity, and age as
static features. Together with the synthetic outcome and treatments, these two benchmark is characterized by 27 time-varying confounders and 3 treatments.

The simulator first constructs two untreated outcome processes,
indexed by $j\in\{1,2\}$:
\begin{equation}
    \widetilde Y_t^j
    =
    \alpha_s^j \operatorname{B\text{-}spline}(t)
    + \alpha_g^j g^j(t)
    + \alpha_f^j f_Y^j(X_t)
    + \epsilon_t.\label{eq:mimic_untreated_outcome}
\end{equation}
The coefficients $\alpha_s^j$, $\alpha_g^j$, and $\alpha_f^j$
weight the respective components. The spline component is
drawn from a mixture of three cubic splines, while
$f_Y^j$ is sampled using a Random Fourier Feature (RFF)
approximation to a Gaussian process.

Three binary treatments, indexed by $l\in\{1,2,3\}$, are
then assigned according to
\begin{equation}
    A_t^l \sim \operatorname{Bernoulli}\left(\sigma(
        \gamma_Y^l Y_{t-1}^{A,l}
        + \gamma_X^l f_Y^l(X_t)
        + b^l)
    \right),\label{eq:mimic_propensity}
\end{equation}
where $\sigma$ denotes the logistic function.
Here, $Y_{t-1}^{A,l}$ is an averaged subset of previous $l$
treated outcomes, and $f_Y^l$ is a again a 
RFF approximation to a Gaussian
process. The coefficients $\gamma_Y^l$ and $\gamma_X^l$
determine the confounding strength of $A^l$, while $b^l$
sets the treatment-specific intercept.

Finally, the simulator adds treatment effects to the untreated
outcome trajectories via
\begin{equation}
    Y_t^j
    =
    \widetilde{Y}_t^j
    +
    \sum_{i=t-\omega^l}^{t}
    \frac{
        \min_{l=1,\ldots,d_a}
        \mathbf{1}_{\{A_i^l=1\}}\,p_i^l\,\beta^{l,j}
    }{
        (\omega^l-i)^2
    }.\label{eq:mimic_treated_outcome}
\end{equation}
For treatment $l$, the parameter
$\beta^{l,j}$ controls the maximum effect on outcome $j$,
and $\omega^l$ specifies its effect window. 

The two simulated outcomes, $Y^1$ and $Y^2$, are governed by different confounders in the real-world covariates; thus we evaluate them as separate benchmarks:
\texttt{mimic\_y1} targets $Y^1$, and \texttt{mimic\_y2} targets $Y^2$. The simulator parameters follow \citet{hess2026igc}, which we refer to for the full specification.

Both \texttt{mimic\_y1} and \texttt{mimic\_y2} have 3-dim treatments at each time step; in both settings, we specify the target policy to 'always apply the second and third treatments', i.e. $A_t = (0,1,1)$ for $\tau$ steps.

\subsubsection{\texttt{mimic\_1t},\;
\texttt{mimic\_1t\_hdp0.5},\;
\texttt{mimic\_1t\_hdp1},\;
\texttt{mimic\_1t\_hdp2}}
To accommodate baselines that require a single treatment variable,
we construct four variants of \texttt{mimic\_y2} with a single binary
treatment. We retain its simulator parameters and outcome-generation
mechanism, but remove treatments $A^1$ and $A^2$ and their contributions
to $Y^2$. The remaining treatment, $A^3$, is assigned according to
the original behavior policy.

The four benchmarks share this data-generating process and differ
only in their target policies. The \texttt{mimic\_1t} benchmark uses
an always-treat policy, enabling comparisons with baselines restricted
to static treatment sequences. The other three benchmarks use
deterministic dynamic policies obtained by shifting the logit of the
behavior policy:
\begin{equation}
    A_t^{3}
    = \mathds{1}
        \left\{
            \gamma_Y^3 Y_{t-1}^{A,3}
            + \gamma_X^3 f_Y^3(X_t)
            + b^3 > \zeta
            \right\}.
\end{equation}
We set $\zeta\in\{0.5,1,2\}$, with the value appended to
\texttt{hdp} in each benchmark name. Larger values of $\zeta$
increase the target treatment probability relative to the behavior
policy for a given history. $\zeta=0$ will yield a deterministic rule that follows the behavior policy's decision boundary.

\subsection{Baseline Implementations}
We compare against nine baselines. \textbf{CRN}~\citep{bica2020estimating},
\textbf{RMSN}~\citep{lim2018forecasting}, \textbf{TE-CDE}~\citep{seedat2022continuous},
\textbf{G-Net}~\citep{li2021g}, and \textbf{CT}~\citep{melnychuk2022causal}
are taken from the CT codebase~\citep{melnychuk2022causal};
\textbf{D3-Net}~\citep{chen2026deep}, \textbf{DeepLTMLE}~\citep{shirakawa2024longitudinal}, and \textbf{DeepACE}~\citep{frauen2023estimating} from the D3-Net
codebase~\citep{chen2026deep}; and \textbf{IGC-Net}~\citep{hess2026igc} from the authors' public implementation.

There is one major change in our implementation of IGC-Net. As published, it estimates the effect of a single \emph{static} treatment sequence shared by every unit. Neither the paper nor the implementation considers dynamic target policies. However, we noticed that it can be readily extended to support such a setting by modifying a step in data preprocessing. Specifically, we extend IGC-Net to our setting by
replacing the hardcoded treatment sequence with the per-trajectory treatment decision of the target policy $\pi$, leaving the architecture and training objective
untouched. This extension goes beyond the published scope of IGC-Net and is
made in its favour: it is what allows IGC-Net to be run on every setting,
including those with a history-dependent target policy.

G-Net supports dynamic treatment regimes in principle: its Monte Carlo rollout simulates future covariates, on which a regime could be evaluated. The released implementation, however, fixes the full treatment sequence before inference and never re-evaluates a policy on the simulated state, so supporting history-dependent regimes would require non-trivial changes to its inference path. We therefore do not run G-Net in this setting. This omission does not affect our conclusions: under static treatment sequences on the same data (Table~\ref{tab:full_benchmark_per_horizons}, \texttt{mimic\_1t}), G-Net already underperforms the leading baselines by a wide margin at every horizon.

CRN, RMSN, TE-CDE, G-Net, and CT are encoder--decoder models that are \emph{trained once per task} on the full trajectory window ($T=15$ steps) and then evaluated at
every horizon $\tau \in \{2,4,6,8,10\}$. D3-Net, DeepLTMLE, DeepACE, and IGC-Net
are based on ICE G-computation, which requires horizon-specific architectures, so a separate model is trained for every (setting, horizon) cell.

Not every baseline runs on every setting. D3-Net, DeepLTMLE, and DeepACE are
hardwired to a single binary treatment (their propensity head, inverse
propensity ratio, and the debiasing objective all assume single-dimensional treatment), so they are evaluated only on the single-treatment MIMIC settings. On the settings with a history-dependent target policy, only IGC-Net, D3-Net, and DeepLTMLE are run: the encoder--decoder models condition on a fixed treatment sequence and cannot
take a target policy that depends on the evolving state. Cells that a method
cannot serve are left blank in the results tables.

\subsection{Hyperparameter Tuning and Evaluation Protocol}\label{appx:baseline_hparams}
For every baseline, we tune its hyperparameters with Optuna~\citep{optuna_2019} for
2 hours on one A100 GPU, using one held-out tuning task drawn from the same
generator as, and disjoint from, the evaluation tasks. The selection criterion
is held-out \emph{factual} RMSE: counterfactual labels are never used for model
or hyperparameter selection. Hyperparameters are tuned once per setting for the
encoder--decoder models, and once per (setting, horizon) cell for ICE G-computation-based methods. For instance, IGC-Net, uses $5 \times 2 = 10$ hours of tuning per setting. The search spaces of every baseline are listed in Table~\ref{tab:hp_grid}.

The selected hyperparameters are then fixed, and a fresh model is trained on
each of the 20 evaluation tasks. Every task contains 2000 trajectories,
split into 1000 training trajectories and 1000 test trajectories. The
baseline is trained on the training half, and scored on the Monte-Carlo simulated CAPO on the test half. This is the same split our model uses: the training half is its in-context set, and the test half its query set, so every method sees identical data.

\begin{table}[t]
\centering
\small
\caption{Hyperparameter search spaces for the baselines. Each space is
searched with Optuna for 2\,h on one A100, selecting by held-out factual RMSE.
$\mathcal{U}$ and $\log\mathcal{U}$ denote uniform and log-uniform ranges;
braces denote categorical choices. 
}
\label{tab:hp_grid}
\begin{tabular}{@{}lll@{}}
\toprule
Method & Hyperparameter & Search space \\
\midrule
\multirow{7}{*}{CRN, TE-CDE, G-Net}
  & learning rate                    & $\log\mathcal{U}[10^{-4},\,10^{-2}]$ \\
  & number of layers                 & $\{1, 2, 3\}$ \\
  & batch size                       & $\{32, 64, 128, 256\}$ \\
  & sequence hidden units            & $\{0.5, 1, 2, 3, 4\}\times$ input dim \\
  & balanced representation size     & $\{0.5, 1, 2, 3, 4\}\times$ input dim \\
  & FC hidden units                  & $\{0.5, 1, 2, 3, 4\}\times$ input dim \\
  & dropout                          & $\mathcal{U}[0,\,0.5]$ \\
\midrule
\multirow{8}{*}{CT}
  & learning rate                    & $\log\mathcal{U}[10^{-4},\,10^{-2}]$ \\
  & number of layers                 & $\{1, 2, 3\}$ \\
  & attention heads                  & $\{1, 2, 3\}$ \\
  & batch size                       & $\{32, 64, 128, 256\}$ \\
  & sequence hidden units            & $\{0.5, 1, 2, 3, 4\}\times$ input dim \\
  & balanced representation size     & $\{0.5, 1, 2, 3, 4\}\times$ input dim \\
  & FC hidden units                  & $\{0.5, 1, 2, 3, 4\}\times$ input dim \\
  & dropout                          & $\mathcal{U}[0,\,0.5]$ \\
\midrule
\multirow{6}{*}{RMSN}
  & learning rate                    & $\log\mathcal{U}[10^{-4},\,10^{-2}]$ \\
  & number of layers                 & $\{1, 2, 3\}$ \\
  & batch size                       & $\{32, 64, 128, 256\}$ \\
  & sequence hidden units            & $\{0.5, 1, 2, 3, 4\}\times$ input dim \\
  & dropout                          & $\mathcal{U}[0,\,0.5]$ \\
  & max.\ gradient norm              & $\{0.5, 1, 2\}$ \\
\midrule
\multirow{7}{*}{D3-Net, DeepLTMLE}
  & learning rate                    & $\log\mathcal{U}[10^{-4},\,10^{-2}]$ \\
  & hidden size                      & $\{16, 32, 64\}$ \\
  & number of layers                 & $\{1, 2, 3\}$ \\
  & attention heads                  & $\{2, 4\}$ \\
  & embedding dimension              & $\{8, 16, 32, 64, 128\}$ \\
  & dropout                          & $\mathcal{U}[0,\,0.5]$ \\
  & propensity loss weight $\alpha$  & $\mathcal{U}[0.5,\,2]$ \\
\midrule
\multirow{7}{*}{DeepACE}
  & learning rate                    & $\log\mathcal{U}[10^{-4},\,10^{-2}]$ \\
  & LSTM hidden size                 & $\{16, 32, 64\}$ \\
  & body hidden size                 & $\{16, 32\}$ \\
  & head hidden size                 & $\{8, 16, 32\}$ \\
  & dropout                          & $\mathcal{U}[0,\,0.5]$ \\
  & propensity loss weight $\alpha$  & $\mathcal{U}[0.5,\,1]$ \\
  & targeting loss weight $\beta$    & $\mathcal{U}[0.5,\,1]$ \\
\midrule
\multirow{8}{*}{IGC-Net}
  & learning rate                    & $\log\mathcal{U}[10^{-4},\,10^{-2}]$ \\
  & number of layers                 & $\{1, 2, 3\}$ \\
  & attention heads                  & $\{1, 2\}$ \\
  & sequence hidden units            & $\{16, 32, 48\}$ \\
  & hidden representation size       & $\{8, 16, 32\}$ \\
  & FC hidden units                  & $\{8, 16, 32\}$ \\
  & dropout                          & $\mathcal{U}[0,\,0.5]$ \\
  & batch size                       & $\{32, 64, 128, 256\}$ \\
\bottomrule
\end{tabular}
\end{table}



\subsection{Runtime measurement}\label{appx:runtime_measurement}
We additionally compare wall-clock deployment cost, using a dedicated single-GPU timing benchmark that reuses each model's already-tuned HPs (so HP tuning time itself is excluded). \methodname{}'s cost is inference-only: forward passes of the single fixed checkpoint over a task's query trajectories, with no gradient update. We repeated five forward passes and report the mean. 
Baselines' cost is training plus inference: each train/inference call is
likewise repeated 5 times and averaged. 

The architectures of CRN, TE-CDE, CT, RMSN, G-Net are not horizon-specific, so they require one training and 5 inferences, one for each target horizon, matching how they are evaluated. Their runtime is divided over five to represent the average runtime for each target horizon. IGC-Net, D3, D-LTMLE, and DeepACE require separate horizon-specific training for each setting, so they are timed
separately at each of the five horizons, and we report the mean of training + inference time across horizons.

\subsection{Real-world MIMIC prediction}
\label{appx:mimic_real}
We follow the cohort construction and data extraction procedures of \citet{hess2026igc}, to which we refer readers for additional details, and define a factual outcome prediction task as follows.

\paragraph{Cohorts and Data}
We extract each patient's hourly trajectory over the first 11 hours. Missing values are imputed by forward- and backward-filling within each trajectory, and patients with any remaining missing outcome or covariate values are excluded. We then sample ten cohorts of $n=3000$ patients each, without replacement within each cohort. Each cohort is split by row order into 1800 training, 600 validation, and 600 test patients. All cohorts and splits are generated once and shared across methods.

We consider two treatments: vasopressor use and mechanical ventilation. The outcome is hourly mean diastolic blood pressure (DBP, mmHg). Time-varying covariates comprise DBP and the 25 vital signs and laboratory measurements selected by \citet{hess2026igc} ($d_L=26$). Static covariates comprise gender, ethnicity, and age.

\paragraph{Prediction task.}
We evaluate multi-step prediction of factual outcomes under the behavior policy. We fix the start hour at $s=5$ and evaluate horizons $\tau\in \{2,3,4,5\}$, corresponding to observed DBP at endpoint hours $\{7,8,9,10\}$. For each cohort and horizon, we report RMSE over the 600 test patients, normalized by the standard deviation of their observed endpoint DBP.

\paragraph{Model Fitting}
IGC-Net is the only baseline that supports both multivariate treatments and dynamic policies, so we only compare with it. We use its architecture as released, 
but replace each patient's target treatment sequence with their own factual treatment sequence when constructing the iterative G-computation targets. This adapts their algorithm to target outcomes under the behavior policy without explicitly estimating that policy. For each cohort and target horizon, we tune its HP for 2h using the training and validation sets, refit it on their union using the selected HPs, and evaluate it on the test set.

For \methodname{}, the training and validation patients jointly form the context set, and the test patients form the query set. We use the context patients' factual treatment sequences as policy demonstrations, providing observed history--action samples from the behavior policy. Neither method explicitly fits a behavior-policy model, avoiding an additional source of estimation error in the comparison.


\subsection{Uncertainty calibration.}\label{appx:uncertainty_calibration}
We measure the calibration of \methodname{} and the strongest baseline, IGC-Net. \methodname{} outputs a full predictive distribution in a single forward pass, from which the equal-tailed $100(1-\alpha)\%$ interval is read off directly. IGC-Net produces only a point estimate, so we follow the paper and obtain intervals by Monte-Carlo dropout (1000 stochastic forward passes at the tuned dropout rate)\footnote{While these MC dropout intervals are the uncertainty-quantification procedure proposed by the original paper, they should not be interpreted as confidence intervals. Bootstrapping could instead be used to construct confidence intervals, but would require resampling the observed data and retraining IGC-Net on every resample, which is computationally prohibitive given its long training time.}. A query trajectory is \emph{covered} if its true CAPO lies inside the interval, and the empirical coverage of a task is the fraction of covered query trajectories. 
Coverage is measured on the datasets from our own prior, \texttt{mimic\_y1}, and \texttt{mimic\_1t}, across horizon $\tau \in \{2,4,6,8,10\}$. Results on each (dataset, horizon) cell is averaged over 20 tasks. 

Major results are reported in Figure~\ref{fig:calibration_curve}. We notice that some sampled tasks have few context trajectories following the target policy over the prediction window. This extreme positivity violation forces any method to extrapolate, and it becomes more frequent as the horizon grows. These cases account for the largest deviations from nominal coverage, in either direction (Figure~\ref{fig:per_task_overlap_by_h_prior}). 



\subsection{Unobserved Confounding}\label{appx:unobserved_confounding} 
We examine how \methodname{} and the two strongest baselines, CT and IGC-Net, degrade when 
unobserved confounders are present.
On \texttt{mimic\_1t}, we draw a patient-level latent severity $U \sim \mathcal{N}(0,1)$ that is hidden from all methods and enters the data-generating process twice by shifting the patient's outcome by $\lambda U$ at every step, 
\begin{equation}
    Y_t^j
    =
    \widetilde{Y}_t^j
    +
    \sum_{i=t-\omega^l}^{t}
    \frac{
        \min_{l=1,\ldots,d_a}
        \mathbf{1}_{\{A_i^l=1\}}\,p_i^l\,\beta^{l,j}
    }{
        (\omega^l-i)^2
    } + \lambda U.\label{eq:mimic_1t_outcome_unobserved_confounder}
\end{equation}
and the logit of the behavior propensity by $\gamma U$ at every step,
\begin{equation}
    A_t^l \sim \operatorname{Bernoulli}\left(\sigma(
        \gamma_Y^l Y_{t-1}^{A,l}
        + \gamma_X^l f_Y^l(X_t)
        + b^l + \gamma U).
    \right),\label{eq:mimic_1t_propensity_unobserved_confounder}
\end{equation}
We fix $\lambda = 0.5$ and vary the confounding strength over $\gamma \in \{1, 2, 4\}$, using the original evaluation without unobserved confounding ($\gamma = \lambda = 0$) as the reference. 

\begin{figure}[t]
    \centering



    \begin{subfigure}{0.95\linewidth}
        \centering
        \includegraphics[width=\linewidth]{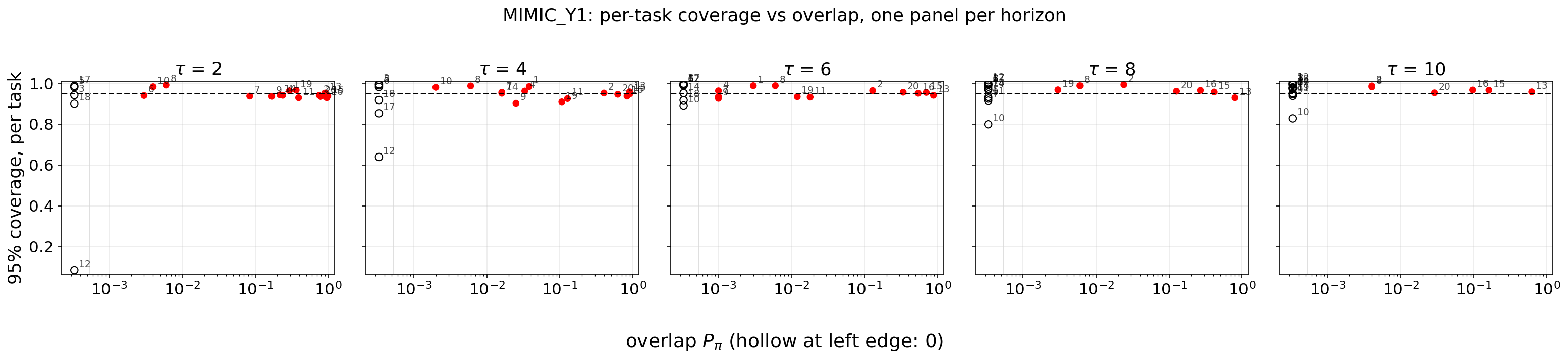}
        \caption{\texttt{mimic\_y1}}
        \label{subfig:per_task_overlap_by_h_y1}
    \end{subfigure}


    \begin{subfigure}{0.95\linewidth}
        \centering
        \includegraphics[width=\linewidth]{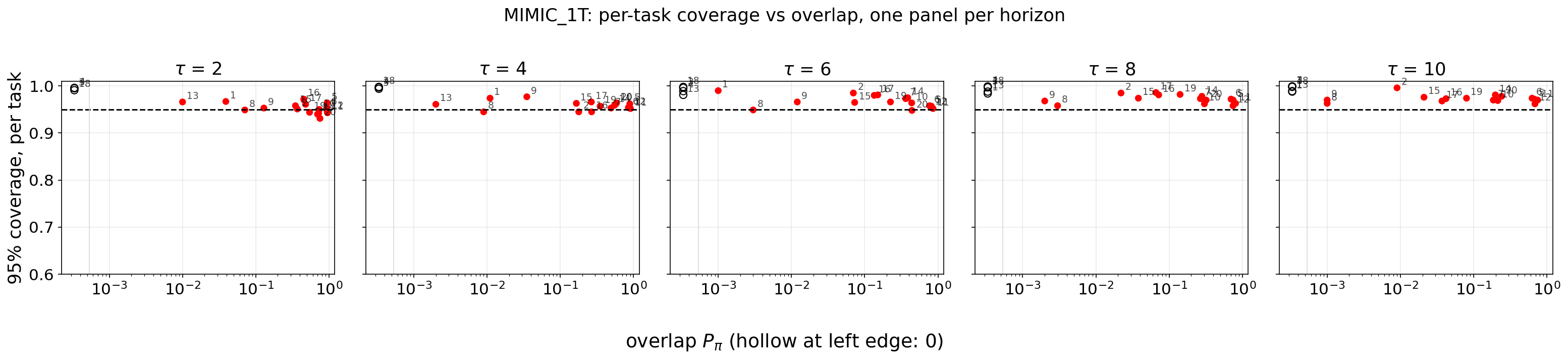}
        \caption{\texttt{mimic\_1t}}
        \label{subfig:per_task_overlap_by_h_prior_1t}
    \end{subfigure}

    \caption{\textbf{Coverage against overlap} Overlap (x-axis) is measured as the proportion of observed trajectories that exactly match the target policy, and coverage (y-axis) is evaluated using nominal 95\% intervals. Hollow points denote tasks with zero empirical overlap. Tasks with positive overlap generally achieve coverage above 0.9, whereas zero-overlap tasks exhibit the largest deviation from nominal coverage, in either direction.}
    \label{fig:per_task_overlap_by_h_prior}
\end{figure}

\begin{table}[t]
\centering
\caption{\textbf{Context sample size under a static and a history-dependent target policy.} RMSE on the query split (size=1000), mean$\pm$SD over 20 tasks ($\downarrow$ better).CT cannot handle dynamic policies and is not run on the history-dependent setting.}
\label{tab:ablation_sample_size}
\resizebox{\linewidth}{!}{
\begin{tabular}{@{}llccccccc@{}}
\toprule
 & & \multicolumn{4}{c}{\texttt{mimic\_extract\_1t}} & \multicolumn{3}{c}{\texttt{mimic\_extract\_1t\_hdp1}} \\
\cmidrule(lr){3-6} \cmidrule(lr){7-9}
Horizon & $n$ context & \textbf{PICLE-PFN} & CT & IGC-Net & D3 & \textbf{PICLE-PFN} & IGC-Net & D3 \\
\midrule
\multirow{3}{*}{$h=2$} & 100 & 0.456$\pm$0.145 & 0.631$\pm$0.195 & \textbf{0.447$\pm$0.079} & 0.596$\pm$0.301 & \textbf{0.403$\pm$0.132} & 0.406$\pm$0.108 & 0.486$\pm$0.131 \\
 & 300 & \textbf{0.382$\pm$0.091} & 0.385$\pm$0.082 & 0.387$\pm$0.077 & 0.498$\pm$0.124 & \textbf{0.347$\pm$0.107} & 0.355$\pm$0.104 & 0.399$\pm$0.111 \\
 & 1000 & 0.377$\pm$0.125 & \textbf{0.374$\pm$0.198} & 0.435$\pm$0.309 & 0.409$\pm$0.086 & \textbf{0.324$\pm$0.100} & 0.337$\pm$0.096 & 0.355$\pm$0.101 \\
\midrule
\multirow{3}{*}{$h=4$} & 100 & \textbf{0.621$\pm$0.204} & 0.838$\pm$0.210 & 0.698$\pm$0.178 & 0.745$\pm$0.217 & \textbf{0.574$\pm$0.181} & 0.661$\pm$0.148 & 0.687$\pm$0.159 \\
 & 300 & \textbf{0.537$\pm$0.133} & 0.578$\pm$0.142 & 0.645$\pm$0.202 & 0.596$\pm$0.133 & \textbf{0.500$\pm$0.151} & 0.539$\pm$0.148 & 0.575$\pm$0.161 \\
 & 1000 & 0.541$\pm$0.217 & 0.574$\pm$0.368 & \textbf{0.524$\pm$0.131} & 0.545$\pm$0.101 & \textbf{0.465$\pm$0.140} & 0.560$\pm$0.133 & 0.528$\pm$0.155 \\
\bottomrule
\end{tabular}}
\end{table}


\section{Comparison with CausalLongPFN}
\label{appx:causallongpfn}

Concurrent work, CausalLongPFN~\citep{zare2026causal}, also uses a PFN to estimate CAPO. However, their method only supports static treatment sequences, whereas \methodname{} supports both static and dynamic treatment policies. More importantly, their approach, in general, do not recover the multi-step CAPO.

For a prediction horizon greater than one, CausalLongPFN recursively applies a one-step outcome predictor. At each step from the start to the end time step, it inserts the predicted outcome mean into the next outcome history while leaving future time-varying covariates masked. This plug-in rollout does not, in general, identify the multi-step CAPO. To see this, consider a simplified setting where the only time-varying confounder is the intermediate outcome itself, yielding the trajectory $O=(B,A_1,Y_2,A_2,Y_3)$.  Support a target treatment sequence is $(a_1,a_2)$. At the second step, CausalLongPFN's one-step predictor models the conditional distribution of $Y_3$ given $(B,A_1,Y_2,A_2)$. Suppose the one-step predictor perfectly recovers the true
conditional mean:
$$
g(b,a_1,y_2,a_2)
=
\mathbb E[Y_3\mid B=b,A_1=a_1,Y_2=y_2,A_2=a_2].
$$

The target CAPO, $\mathbb E[Y_3^{a_1,a_2}\mid B=b]$, satisfies
$$
\mathbb E[Y_3^{a_1,a_2}\mid B=b]
=
\mathbb E[g(b,a_1,Y_2,a_2)\mid B=b,A_1=a_1].
$$

The critical flaw in their one-step mean rollout approach is that, it inserts the first-step predicted mean
$\widehat y_2=\mathbb E[Y_2\mid B=b,A_1=a_1]$
and returns $g(b,a_1,\widehat y_2,a_2)$. In general,
$$
\mathbb E[g(b,a_1,Y_2,a_2)\mid B=b,A_1=a_1]
\neq
g(b,a_1,\mathbb E[Y_2\mid B=b,A_1=a_1],a_2).
$$
Thus, even with exact one-step conditional means, mean
plug-in rollout does not generally recover the multi-step CAPO.

Additionally, the released CausalLongPFN checkpoint supports at most 10 time-varying covariates and a maximum prediction horizon of five. Our MIMIC benchmark contains 27 time-varying covariates and is therefore incompatible with its input interface. Consequently, among our benchmarks, we can evaluate CausalLongPFN only on \texttt{cancer} and restrict the comparison to $\tau_{\max}=5$. This benchmark is less sensitive to the flaw in their one-step mean rollout approach, because it only has one time-varying confounder, which is the outcome itself in the intermediate step. Table~\ref{tab:compare_causallongpfn} shows that CausalLongPFN achieves lower RMSE at $\tau\in\{1,2\}$, whereas \methodname{} achieves lower RMSE at longer horizons. This pattern is consistent with the limitation of mean plug-in rollout discussed above.

CausalLongPFN also requires substantially longer inference time at every evaluated horizon. Two things of its method contribute to this cost. First, multi-step prediction requires multiple forward pass because CausalLongPFN needs sequential one-step model evaluations, whereas \methodname{} estimates all supported horizons in a single forward pass. Second, its context tokens attend to the query token, making their representations query-dependent. Consequently, the full context set is re-encoded for each query trajectory at every rollout step, rather than being encoded once and reused across queries.

\begin{table}[t]
\centering
\caption{Our PFN vs.\ CausalLongPFN on \texttt{cancer}: normalized RMSE (mean $\pm$ std over 20 tasks) and inference runtime per task (1000 context / 1000 query trajectories, one A100).}
\label{tab:compare_causallongpfn}
\begin{tabular}{c cc cc}
\toprule
 & \multicolumn{2}{c}{Normalized RMSE $\downarrow$} & \multicolumn{2}{c}{Runtime (s) $\downarrow$} \\
\cmidrule(lr){2-3} \cmidrule(lr){4-5}
Horizon $\tau$ & \methodname{} & CausalLongPFN & \methodname{} & CausalLongPFN \\
\midrule
1 & 0.275 $\pm$ 0.026 & \textbf{0.195} $\pm$ 0.020 & \textbf{0.039} & 6.67 \\
2 & 0.308 $\pm$ 0.030 & \textbf{0.278} $\pm$ 0.025 & \textbf{0.051} & 12.49 \\
3 & \textbf{0.298} $\pm$ 0.037 & 0.306 $\pm$ 0.026 & \textbf{0.065} & 18.80 \\
4 & \textbf{0.264} $\pm$ 0.038 & 0.298 $\pm$ 0.030 & \textbf{0.085} & 24.75 \\
5 & \textbf{0.216} $\pm$ 0.032 & 0.278 $\pm$ 0.035 & \textbf{0.100} & 31.18 \\
\bottomrule
\end{tabular}
\end{table}

\section{Complexity}\label{appx:complexity} 
Let $N$ be the number of trajectories (context and query,
processed jointly), $T$ the maximum length, $\tau_{\max}\le T$ the maximum prediction
horizon, and $d=\max\{d_L,d_S\}>d_A$ the covariate dimension. Feature pooling
self-attends among the covariates of each time step, costing $O(NTd^2)$. In the
observed stream, causal attention and row attention cost $O(NT^2)$ and
$O(TN^2)$. Each trajectory carries
$K=T\tau_{\max}-\tau_{\max}(\tau_{\max}-1)/2=\Theta(T\tau_{\max})$ readout tokens, one per
(start, endpoint) window; their cross-attention to the observed stream costs $O(NKT)$
and their row attention $O(KN^2)$. In total,
\[
O\!\left(NTd^2 + NT^2 + TN^2 + NT^2\tau_{\max} + T\tau_{\max}N^2\right)
= O\!\left(NTd^2 + T\tau_{\max}N^2\right) \quad\text{for } T\ll N,
\]
which is quadratic in the number of trajectories, linear in the sequence length for a
fixed horizon cap, and at most $O(NTd^2+T^2N^2)$ for the full triangle $\tau_{\max}=T$.
Since row attention uses only the $N_c$ context trajectories as keys, its exact cost
is $O(T\tau_{\max}NN_c)$, i.e.\ linear in the number of queries for a fixed context.

\end{document}